\documentclass[10pt]{article}
\usepackage{fullpage,parskip}
\usepackage{commath}

\usepackage{amsmath, amsthm,amssymb,mathtools}

\usepackage{url}
\usepackage{bm}
\usepackage{xcolor}
\usepackage{enumitem}
\usepackage{hyperref}
\hypersetup{
	colorlinks   = true, 
	urlcolor     = blue, 
	linkcolor    = blue, 
	citecolor   = blue 
}
\usepackage{cleveref}
\usepackage[utf8]{inputenc} 
\usepackage[T1]{fontenc}    
\usepackage{url}            
\usepackage{booktabs}       
\usepackage{amsfonts}       
\usepackage{nicefrac}       
\usepackage{microtype}      
\usepackage{graphicx}
\usepackage[multiple]{footmisc}
\usepackage{verbatim}
\usepackage{graphicx}
\usepackage{natbib}
\usepackage{mathrsfs}

\usepackage[utf8]{inputenc}
\usepackage{color, colortbl}

\usepackage[sc,osf]{mathpazo}

\definecolor{lightGray}{gray}{0.9}

\newcommand{\bbR}{\mathbb{R}}
\newcommand{\bbE}{\mathbb{E}}

\newcommand{\bI}{\bm{I}}
\newcommand{\bQ}{\bm{Q}}
\newcommand{\bR}{\bm{R}}
\newcommand{\bS}{\bm{S}}
\newcommand{\bX}{\bm{X}}

\newcommand{\bq}{\bm{q}}
\newcommand{\bw}{\bm{w}}

\newcommand{\R}{\mathbb{R}}

\newcommand{\bx}{\bm{x}}
\newcommand{\bz}{\bm{z}}

\newcommand{\by}{\bm{y}}
\renewcommand{\bz}{\bm{z}}
\newcommand{\bv}{\bm{v}}

\newcommand{\eps}{\varepsilon}
\newcommand{\Lip}{\mathrm{Lip}}

\newcommand{\OPT}{\mathsf{OPT}}
\newcommand{\loss}{\mathcal{L}}

\newcommand{\fxw}[2]{f(\langle \bx_{#1}, #2 \rangle)}
\newcommand{\Net}{\mathcal{N}}
\renewcommand{\P}{\mathcal{P}}
\newcommand{\E}{\mathbb{E}}
\newcommand{\bbP}{\mathbb{P}}
\DeclareMathOperator*{\Diam}{Diam}
\newcommand{\reg}{\mathsf{reg}}

\newcommand\blfootnote[1]{%
  \begingroup
  \renewcommand\thefootnote{}\footnote{#1}%
  \addtocounter{footnote}{-1}%
  \endgroup
}

\DeclareMathOperator*{\argmin}{arg\,min}
 
\newtheorem{theorem}{Theorem}
\newtheorem{lemma}[theorem]{Lemma}

\newtheorem{corollary}[theorem]{Corollary}
\newtheorem{definition}[theorem]{Definition}

\newtheorem*{rep@theorem}{\rep@title}
\newcommand{\newreptheorem}[2]{%
	\newenvironment{rep#1}[1]{%
		\def\rep@title{#2 \ref{##1}}%
		\begin{rep@theorem}}%
		{\end{rep@theorem}}}
\makeatother
\newreptheorem{theorem}{Theorem}

\title{Agnostic Active Learning of Single Index Models with Linear Sample Complexity}

\date{}
  \usepackage{nth}
  \usepackage{intcalc}

  \newcommand{\cAAAI}[1]{AAAI\ Conference\ on\ Artificial (AAAI)}

\author{Aarshvi Gajjar \thanks{Equal Contribution} \footnote{New York University (\href{mailto:aarshvi@nyu.edu}{\texttt{aarshvi@nyu.edu}}, \href{mailto:chinmay.h@nyu.edu}{\texttt{chinmay.h@nyu.edu}}, 
\href{mailto:cmusco@nyu.edu}{\texttt{cmusco@nyu.edu}})}  
\and Wai Ming Tai \footnotemark[1]\footnote{Nanyang Technological University (\href{mailto:waiming.tai@ntu.edu.sg}{\texttt{waiming.tai@ntu.edu.sg}}, \href{mailto:yili@ntu.edu.sg}{\texttt{yili@ntu.edu.sg}})}
\and Xingyu Xu \footnotemark[1] \footnote{Carnegie Mellon University (\href{mailto:xingyuxu@andrew.cmu.edu}{\texttt{xingyuxu@andrew.cmu.edu}})}
\and Chinmay Hegde\footnotemark[2] \and \and Yi Li \footnotemark[3] \and Christopher Musco \footnotemark[2]
}

\date{\phantom{x}}

\begin{document}
\maketitle

\begin{center}\vspace{-1cm}\today\vspace{0.5cm}\end{center}
\begin{abstract}%

\blfootnote{\newline \centering 
 Accepted for presentation at the
Conference on Learning Theory (COLT) 2024}

We study active learning methods for single index models of the form $F({\bm x}) = f(\langle {\bm w}, {\bm x}\rangle)$, where $f:\mathbb{R} \to \mathbb{R}$ and ${\bx,\bm w} \in \mathbb{R}^d$. In addition to their theoretical interest as simple examples of non-linear neural networks, single index models have received significant recent attention due to applications in scientific machine learning like surrogate modeling for partial differential equations (PDEs). Such applications require sample-efficient active learning methods that are robust to adversarial noise. I.e., that work even in the challenging agnostic learning setting.
  
    We provide two main results on agnostic active learning of single index models. First, when $f$ is known and Lipschitz, we show that $\tilde{O}(d)$ samples collected via {statistical leverage score sampling} are sufficient to learn a near-optimal single index model. Leverage score sampling is simple to implement, efficient, and already widely used for actively learning linear models. Our result requires no assumptions on the data distribution, is optimal up to log factors, and improves quadratically on a recent ${O}(d^{2})$ bound of \cite{gajjar2023active}. Second, we show that $\tilde{O}(d)$ samples suffice even in the more difficult setting when $f$ is \emph{unknown}. Our results leverage tools from high dimensional probability, including Dudley's inequality and dual Sudakov minoration, as well as a novel, distribution-aware discretization of the class of Lipschitz functions.
\end{abstract}

\section{Introduction}
The goal in active learning is to effectively fit a model based on a small amount of labeled data selected from a vast pool of unlabeled data.
Active learning finds applications in settings where labeling data is expensive. For example, a common task in scientific machine learning is to learn parameter-to-solution maps for a parametric partial differential equation (PDE) \citep{cohen2015approximation}. Obtaining each label involves the high-cost numerical solution of the PDE, so learning based on few labels is necessary for efficiency. In other settings, such as experimental science or sensor placement, each label requires running a physical experiment or placing a physical device, so is even more expensive \citep{AlexanderianPetraStadler:2016}.

In this paper, we study active learning algorithms that address the challenging \emph{agnostic setting}, aka the adversarial noise setting. Focusing on least squares error for concreteness, given a set of examples $\bx_1,\bx_2,\ldots,\bx_n \in \bbR^{d}$, query access to an arbitrary target vector $\by=(y_1,\ldots,y_n) \in \bbR^n$, and some function class $\mathcal{H}$, our goal is to use a small number of queries to find $\tilde{h}\in \mathcal{H}$ such that: 
\begin{align*}
\sum_{i=1}^n (\tilde{h}(\bx_i) - y_i)^2 \leq \min_{h\in\mathcal{H}}\sum_{i=1}^n ({h}(\bx_i) - y_i)^2 + \Delta,
\end{align*}
for some error parameter $\Delta$. This setting is agnostic because we make \emph{no assumptions} on how $y_i$ is related to $\bx_i$. E.g., in contrast to most work on statistical experimental design \citep{Pukelsheim:2006}, contextual bandits \citep{XuHondaSugiyama:2018}, and other formulations of active learning \citep{BalcanBeygelzimerLangford:2006,Kaariainen:2006}, we do not assume a ``realizable'' setting where $y_i = h^*(\bx_i)$ or $y_i = h^*(\bx_i) + \eta_i$ for some ground truth $h^*\in \mathcal{H}$ and mean-zero random noise $\eta_i$. 
We want to compete with the best approximation to $y$ in $\mathcal{H}$, \emph{even if that approximation is poor.} The agnostic setting is important in scientific applications where model misspecification is expected: a simple and efficient machine learning model is being used to approximate a complex physical process or function.

Algorithms for agnostic active learning have received significant recent attention.
Even for the simplest possible setting of relative-error linear regression, where $\mathcal{H}$ contains linear functions of the form $h(\bx) = \langle \bx, \bw\rangle$ and $\Delta = \eps \cdot \min_{h\in\mathcal{H}}\sum_{i=1}^n ({h}(\bx_i) - y_i)^2$ for $\eps \in (0,1)$, the optimal active sample complexity was only recently settled to be $\Theta(d/\eps)$ in the agnostic setting \citep{chen2019active}. Other recent work establishes sample complexity bounds for $\ell_p$ linear regression \citep{ChenDerezinski:2021,MuscoMuscoWoodruff:2022}, logistic regression \citep{mai2021coresets,MunteanuSchwiegelshohnSohler:2018}, polynomial regression \citep{ShimizuChengChristopher-Musco:2024}, and kernel learning \citep{ErdelyiMuscoMusco:2020}.

\subsection{Single Index Models}
In contrast to the majority of this prior work, and motivated by applications in scientific machine learning, in this paper we study agnostic active learning methods for \emph{non-linear function families}. In particular, we are interested in the class of \emph{single index models} of the form: 
\begin{align*}
    h(\bx) = f(\langle \bx,\bw\rangle), 
\end{align*}
where $f: \R\rightarrow \R$ is either a known non-linearity (ReLU, sigmoid, etc.) or a learnable function.
Also known as ``ridge functions'' or ``plane waves'', single index models play an important role in many estimation problems and have been extensively studied in statistics \citep{hristache2001direct, hardle2004nonparametric, dalalyan2008new}. They are effective at modeling physical phenomena, so have also been applied e.g. to PDE surrogate modeling \citep{Cohen2011,HokansonConstantine:2018,BigoniMarzoukPrieur:2022}. As discussed, agnostic learning is important in such applications due to model misspecification, and agnostic active learning is already widely used for fitting simpler functions classes like polynomials \citep{HamptonDoostan:2015,RauhutWard:2012}.

Moreover, single index model are a natural first step towards understanding active learning for neural networks more broadly, an important goal given the increasing importance of neural networks in approximating parameter-to-solution maps \citep{geist2021numerical,bhattacharya2021model,kutyniok2022theoretical} and quantities of interest \citep{Tripathy_2018,khoo_lu_ying_2021,Zhang_2019,olearyroseberry2021derivativeinformed,cardenas2023csml} in scientific ML.

While there has been prior work in the realizable or i.i.d. mean zero noise setting \citep{Cohen2011, fornasier2012learning, 6288306} the first result on actively learning single index models  under adversarial agnostic noise is due to \cite{gajjar2023active}. 
That work shows:
\begin{theorem}[Theorem 1 from \cite{gajjar2023active}]
\label{thm:aistats_result}
Let $f$ be a fixed $L$-Lipschitz function, let $\bX\in \R^{n\times d}$ be a data matrix, and let  $\bw^\star = \arg\min_{\bw} \|f(\bX {\bw}) - \by\|_2^2$. There is an algorithm that, for any $\eps \in (0,1)$, observes $\tilde{O}\left(d^2\cdot \frac{L^8}{\eps^4}\right)$ entries of $\by$ and returns $\hat{\bw}$ such that, for a fixed constant $C>0$,
\begin{align*}
    \|f(\bX \hat{\bw}) - \by\|_2^2 \leq C \|f(\bX {\bw}^\star) - \by\|_2^2 + \eps \|\bX{\bw}^\star\|_2^2,
\end{align*}
with high probability. Above, $f(\bX {\bw})$ denotes the entrywise application of $f$ to the vector $\bX {\bw}$.

\end{theorem}
The above guarantee is strong since it gives relative error with respect to the \emph{best} approximation to the target $\by$, plus a small additive term depending on $\|\bX{\bw}^*\|_2^2$. \cite{gajjar2023active} show that the additive term is necessary:  pure multiplicative error requires $\Omega(2^d)$ samples in the worst case. 

\subsection{Comparison to Non-active Supervised Learning}
\label{sec:compare_to_non_active}
We highlight that, for the strong guarantee of \Cref{thm:aistats_result}, there is an inherent gap between active and non-active supervised learning. 

In particular, consider the uniform distribution over the rows of $\bX$. \Cref{thm:aistats_result} ensures that $\text{poly}(d)$ actively collected labels are needed to nearly minimize the expected squared error over all single index models. In contrast, consider an extreme case where $\bX$ has $d$ rows equal to the standard basis, and all other rows are zero. Even when $f$ is the identity function, we \emph{must} observe target values for those $d$ rows if we want expected squared error competitive with $\bw^*$. For large $n$, a $\text{poly}(d)$ sample algorithm that uniformly samples rows of $\bX$ will only observe those $d$ labels with arbitrarily small probability, so will fail to obtain a bound as strong as \Cref{thm:aistats_result}. 

As such, while there is a rich line of work on learning $d$-dimensional single index models in the agnostic setting with $\text{poly}(d)$ or even fewer samples \citep{gollakota2023agnostically, goel2016reliably, diakonikolas2020approximation, frei2020agnostic,diakonikolas2022learning}, that work inherently requires strong assumptions on $\bX$ or $\by$, or provides a weaker notion of near optimal learning than \Cref{thm:aistats_result}.\footnote{As a concrete example, \cite{gollakota2023agnostically} assumes $\by$ is bounded by $1$ and provides additive error $\eps$ on the average squared error. So, missing the $d$ basis rows in our hard example is acceptable whenever $n \geq d/\eps$.} 
Such assumptions are reasonable in many settings. In others, however, experimental and theoretical evidence shows that active learning can lead to significant improvements in sample complexity \citep{DerezinskiWarmuthHsu:2018,ShimizuChengChristopher-Musco:2024}. For example, even for polynomial regression on one-dimensional data drawn uniformly from an interval, where $f$ in the identity and $\bX$'s columns contain a basis for the degree $d$ polynomials, there is a well-known gap of $O(d^2)$ vs. $O(d)$ samples for non-active vs. active methods \citep{CohenDavenportLeviatan:2013, CohenMigliorati:2017}. Notably, such a basis will have widely varying row norms, as in the hard case discussed above. 

Since our goal is to study theoretical guarantees that inspire and motivate better active learning algorithms, in this work we are primarily interested in guarantees like \Cref{thm:aistats_result} with no distributional assumptions on $\bX$ or functional assumptions on $\by$.

\subsection{Our Results} 
In this paper, we significantly improve on the results of \cite{gajjar2023active} in two ways. First, when $f$ is a known $L$-Lipschitz function, we obtain a sample complexity bound with a \emph{linear} dependence on $d$ and a quadratically improved dependence on $\eps$ and $L$\footnote{A paper accepted to ICLR 2024 \citep{huang2023one} also claims a linear dependence on $d$. However, there was an unfixable flaw in an earlier version of the proof, which has been communicated to the authors.}:
\begin{theorem}
\label{thm:main_known_f}
    Let $f$ be a fixed $L$-Lipschitz function with $f(0) = 0$, let $\bX\in \R^{d\times n}$ be a data matrix, and let  $\bw^\star = \arg\min_{\bw} \|f(\bX {\bw}) - \by\|^2$. There is an algorithm that, for any $\eps \in (0,1)$, observes $\tilde{O}\left(d\cdot \frac{L^4}{\eps^2}\right)$ entries of $\by$ and returns $\hat{\bw}$ such that, for a fixed constant $C>0$, with high probability,
\begin{align*}
    \|f(\bX \hat{\bw}) - \by\|^2 \leq C\|f(\bX {\bw}^*) - \by\|^2 + \eps \|\bX{\bw}^*\|^2,
\end{align*}

\end{theorem}
As noted in \cite{gajjar2023active}, assuming that $f(0) = 0$ is without loss of generality, since we can always shift $\by$ by a fixed constant before fitting.
It is easy to check that our near-linear dependence on $d$ is optimal up to log factors, since $\Omega(d)$ samples are required even when $f(t) = t$. \Cref{thm:main_known_f} establishes that, in terms of dimension, there is no gap between the active agnostic single index learning problem and linear regression. An interesting question for future work is if the dependence on $\eps$ can be improved to linear, as is possible for linear regression \citep{chen2019active}.

\paragraph{Leverage score sampling.} As in \cite{gajjar2023active}, \Cref{thm:main_known_f} is based on collecting samples via \emph{statistical leverage score sampling}. A label for the $i^\text{th}$ row of $\bX$ is selected with probability proportional to the $i^\text{th}$ leverage score $\bx_i^T(\bX^\top\bX)^{-1}\bx_i$. Also known as coherence motivated, Christoffel function, or effective resistance sampling \citep{HamptonDoostan:2015, adcock2023cs4ml,SpielmanSrivastava:2011}, leverage score sampling yields near-optimal active learning bounds for least squares regression, polynomial regression, kernel learning, and variety of other  ``linear'' problems \citep{Sarlos:2006,CohenMigliorati:2017,avron2019universal}. Leverage score sampling has two major advantages. First, it is computationally efficient: all leverage scores can be computed exactly in $O(nd^2)$ time, or approximately using faster randomized methods \citep{MahoneyDrineasMagdon-Ismail:2012}. Second, sampling is done in a completely non-adaptive way: the choice of which indices to label does not depend on prior labels collected. This allows for fully parallel data collection.

\paragraph{Technical contributions in Theorem~\ref{thm:main_known_f}.}Our improved analysis of leverage score sampling in \Cref{thm:main_known_f} requires two new contributions: 
\begin{enumerate}
    \item First, we provide an improved ``subspace embedding'' result for single index functions (\Cref{lem:fixed_concentration}), which shows that leverage score sampling preserves the $\ell_2$ distance between any two vectors of the form $f(\bX\bw_1)$ and $f(\bX\bw_2)$. When $f$ is the identity function, an optimal subspace embedding result for leverage score sampling follows from standard matrix Chernoff bounds \citep{Tropp:2015}. However, when $f$ is non-linearity, an analysis from first principles is required. \cite{gajjar2023active} employs an $\eps$-net argument, which we improve using more powerful tools from high-dimensional probability, including Dudley's inequality and dual Sudakov minoration. This improvement accounts for our linear vs. quadratic dependence on the dimension $d$. 
    \item Second, to improve on the $\eps$ and $L$ dependence, we show a more efficient translation from our subspace embedding result to the active learning problem by analyzing a  natural regularized loss minimization procedure for finding $\hat{\bw}$. The full proof is discussed in detail in \Cref{sec:main_results}.
\end{enumerate}
Our other improvement on \Cref{thm:aistats_result} from \cite{gajjar2023active} is that we are able to extend the result to the more challenging setting where $f$ is an \emph{unknown} Lipschitz function, and can be optimized as part of the training procedure. This setting is well-motivated in computational science, where $f$ is typically parameterized as a piecewise constant or polynomial function, and has been studied in the realizable or i.i.d. noise setting \citep{Cohen2011,6288306,HemantCevher:2012}. To the best of our knowledge, we provide the first result in the agnostic setting:

\begin{theorem}\label{thm:unknown_f}  Let $\Lip_L$ denote the set of all $L$-Lipschitz, real-valued scalar functions $f$ with $f(0) = 0$. Let $\bX\in \R^{d\times n}$ be a data matrix, and let  $(f^\star, \bw^\star) = \arg\min_{f \in \Lip_L,~\bw\in\bbR^d} \norm{f(\bX \bw) - \by}^2$.
There is an algorithm that, for any $\eps \in (0,1)$, observes $\tilde{O}\left(d\cdot \frac{L^4}{\eps^2}\cdot \log^2 n\right)$ entries of $\by$ and returns a pair $(\hat f, \hat{\bw})$, where $\hat \bw \in \R^d$ and $\hat f \in \Lip_L$. With high probability, for a fixed constant $C>0$, this pair satisfies:
    \begin{align*}
        \norm{\hat{f}(\bX \hat{\bw}) - \by}^2 \le C \norm{f^\star(\bX \bw^\star) - \by}^2 + \eps \norm{\bX \bw^\star}^2.
    \end{align*}
\end{theorem}
\Cref{thm:unknown_f} nearly matches \Cref{thm:main_known_f} for the case when $f$ is known, except for a mild logarithmic dependence on $n$. 

\paragraph{Technical contributions in Theorem~\ref{thm:unknown_f}.}
Proving the result requires significant additional work beyond \Cref{thm:main_known_f}. In particular, a natural approach would be to construct an $\eps$-net $\mathcal{N}_\eps$ over all $L$-Lipschitz functions such that e.g. for all $f\in \Lip_L$, there is some $\tilde{f}\in \mathcal{N}_\eps$ such that $\|f(\bX\bw) - \tilde{f}(\bX\bw)\|_2 \leq \eps\|\bX\bw\|_2$. The challenge in doing so is that $f(\bX\bw)$ is a length $n$ vector, our net would inherently requires a discretization of $\Lip_L$ with coarseness depending on $1/n$. This would introduce undesirable polynomial dependencies on $n$ into our sample complexity. 

We avoid this issue by building a sampling-aware discretization, in such a way that $|f(\langle \bx_j, \bw\rangle) - \tilde f(\langle \bx_j, \bw\rangle)| \le \eps |\langle \bx_j, \bw\rangle|$ for most of the indices $j$ sampled in $\by$. Roughly speaking, our discretization ensures a finer approximation to $f$ for inputs close to $0$, and is coarser for values further from $0$. It takes significant effort to bound the size of discretization (and even to figure out what is the proper measure of ``size''), which requires leveraging a diverse set of techniques including generic chaining for Bernoulli processes, dyadic decompositions, and a construction of an embedding to simpler spaces to control the size of our discretization.

\subsection{Additional Discussion of Related Work}

As discussed, actively learning single index models has been studied in prior literature, although not in the agnostic setting \citep{Cohen2011,6288306}. Prior work also considers more challenging ``multi-index'' models of the form $\sum_{i=1}^k f_i(\langle \bw_i,\bx\rangle)$ \citep{fornasier2012learning,HemantCevher:2012}. Understanding the multi-index problem in the agnostic setting is an interesting direction for future work. In prior work, various assumptions on the non-linearity $f$ have been considered, including that $f$ is a low-degree polynomial \citep{chen2020towards} and that $f$ has bounded derivatives of high order \citep{Cohen2011, HemantCevher:2012}. In line with recent work \citep{gajjar2023active,gollakota2023agnostically}, we make a Lipschitz assumption because it is simple, yet captures most fixed non-linearities commonly used in neural networks, like ReLUs and sigmoids.
A different but related line of work considers the setting where $\by$ is assumed to be generated by a single index model, and the goal is to approximate the target with few (non-active) samples, and in a computationally efficient way, by fitting it with a shallow neural network \citep{bietti2022learning, kakade2011efficient,mousavi2022neural,abbe2022merged, pmlr-v75-dudeja18a,damian2022neural,damian2023smoothing}. Such results are motivated by an effort to understand the representation capability and optimization properties of neural networks. One recent result in this setting considers an agnostic guarantee, where the goal is to find an approximation to $\by$ competitive with the best single neuron approximation \citep{gollakota2023agnostically}. However, as detailed in \Cref{sec:compare_to_non_active}, that work inherently requires stronger assumptions than ours, since strong, distribution-independent guarantees of the form provided by \Cref{thm:aistats_result,thm:main_known_f,thm:unknown_f} are not possible without active learning.

\paragraph{Remark on computational efficiency.} In contrast to some of the work above, we \emph{do not} analyze computational efficiency, only sample complexity. However, our active sampling algorithm is computationally efficient since it simply requires computing leverage scores. Moreover, in the fixed $f$ setting, after collecting samples, we prove that it suffice to simply fit a single index model to those samples using a regularized $\ell_2$ loss. As confirmed experimentally in \cite{gajjar2023active}, this can be done easily and efficiently in practice using gradient descent methods. Nevertheless, while beyond the scope of our work, we think formally analyzing the computation efficiency of single index learning methods in the active, agnostic setting is a nice direction for future work.

Beyond research on single index models, we also note that there has been a significant recent work on minimising objectives that can be expressed as $\min_{\bw}\norm{f(\bX \bw - \by)}_2$, where $f$ is a non linearity \citep{jambulapati2023sparsifying, mai2021coresets, MunteanuSchwiegelshohnSohler:2018, MuscoMuscoWoodruff:2022}. While this problem differs in important ways from ours, given that $\by$ is \emph{inside} the non-linearity, methods like leverage score sampling have also proven valuable in this setting. For example, \citet{jambulapati2023sparsifying} employs tools similar to those used in the proof of Theorem~\ref{thm:main_known_f} to obtain bounds for functions $f$ with a natural ``auto-Lipschitz`` property. However, we note that this property is incomparable to our $L$-Lipschitz assumption on $f$; for instance the $1$-Lipschitz ReLU function is not $L$-auto-Lipschitz for any finite $L$.

\section{Notation}\label{sec:notation}
For a natural number $n$, we let $[n]$ denote the set $\{1,2,\ldots,n\}$. For a vector $\bx$ in $\R^d$ with entries $\bx_1, \ldots, \bx_d$, $\norm{\bx} = (\sum_{i=1}^d \bx_i^2)^{1/2}$ denotes its $\ell_2$ norm. We use $\Lip_L$ to represent the class of $L$-Lipschitz functions on $\R$ vanishing at $0$, i.e., $\Lip_L = \{f\in\mathcal C(\bbR): f(0)=0,~ \abs{f(x_1) - f(x_2)} \le L \abs{x_1-x_2}, \forall x_1,x_2 \in \R\}$.

We extend the notation of $f(\cdot)$ to $n$-dimensional vectors: for $\bz\in\bbR^n$, denote $f(\bz) \in \bbR^n$ as the entrywise application of $f$ to $\bz$, i.e. $f(\bz) = (f(\bz_1),f(\bz_2),\ldots, f(\bz_n))$. We denote the $i$-th  standard basis vector as $\mathbf{e}_i$. 
The Euclidean ball of radius $R$ centered at $\bx\in\bbR^d$ is denoted by $B_{\bx}(R)$. 
In the case where the ball is centered at the origin, we simply use $B(R)$. 

Throughout the paper, $c$ and $C$ will denote positive universal constants that may vary upon each occurrence. The notation $\Tilde{O}(m)$ denotes $O(m\log^c m)$ for a fixed constant $c$. Moreover $a \lesssim b$ means that there exists a positive constant $C>0$ such that $a\le Cb$. 

\section{Preliminaries}\label{sec:prelims}
\label{sec:pre}
Our goal is to find a single index model that best fits a given set of data points, $(\bx_j, \by_j)$ for $j \in [n]$. In particular, for some class of functions $\mathcal{F}$, we wish to solve the problem
\begin{equation}
\label{eqn:ls}
 \min_{{f\in\mathcal F , \bw\in\bbR^d}}~\loss(f, \bw), 
 \quad \text{where }\loss(f, \bw)\coloneqq \sum_{j=1}^n \left|f(\langle\bw, \bx_j\rangle) - y_j\right|^2 = \|f(\bX\bw) - \by\|^2.
\end{equation}
In this work, we consider the case when $\mathcal{F} = \{f\}$, containting just a single known function $f$, as well as the case when it is the set of all Lipschitz functions with $f(0) = 0$.



Suppose that $(f^\star, \bw^\star)$ minimize
the loss function $\loss$ over $f\in\mathcal F, \bw\in\bbR^d$. We  define
\[
\OPT(\mathcal F) \coloneqq \min_{{f\in\mathcal F , \bw\in\bbR^d}}\loss(f, \bw) = \loss(f^\star, \bw^\star).
\]
We measure the accuracy of an approximate solution to ~\eqref{eqn:ls} 
by quantifying the difference between its loss and the optimal loss. Specifically, for a given solution $(f,\bw)$, we define accuracy as follows: 

\begin{definition} [$\eps$-accurate solution]
\label{def:accuracy}
Fix some sufficiently large constant $C>0$. Given $\eps>0$, a pair $(f, \bw)$ with $f\in\mathcal F$, $\bw\in\bbR^d$ is said to be an $\eps$-accurate solution to the problem~\eqref{eqn:ls}, if
\[
\loss(f,\bw)\le C \cdot \OPT(\mathcal F) + \eps\|\bX\bw^\star\|^2.
\]
\end{definition}
This notion of accuracy was also used in  \citet{gajjar2023active}. Similar notions also appear in work on agnostic  active learning in other contexts \citep{avron2019universal}.
\subsection{Subsampled regression}
\label{subsec:subsampled-regression}
As in \citet{gajjar2023active}, we collect data in an active way via \emph{importance sampling}. Every row of $\bX$ is assigned a score, and rows are sampled with probability proportional to those scores. Target values in $\by$ only need to be observed for rows that get sampled. Concretely, we use statistical leverage scores, which have been widely applied for active linear regression (see e.g. \citet{chen2019active, mahoney2011randomized}. Intuitively, leverage scores measure how influential a row is in forming the column space of the data matrix, $\bX$. Formally, they are definied as follows:

\begin{definition}[Statistical leverage score] The leverage score of the $j$-th row $\bx_j$ of a matrix $\bX \in \R^{n \times d} $ is defined as $\tau_j(\bX) \coloneqq \bx_j^\top (\bX^\top \bX)^{-1}\bx_j = \sup_{\bw \in \R^d} \frac{\langle  \bx_j,\bw  \rangle^2}{\norm{\bX \bw}_2^2}$.  
\end{definition}

\paragraph{Sampling process.} We consider a sample-with-replacement variant of leverage score sampling. Let $p$ denote a probability distribution over $[n]$ where $j \in [n]$ is assigned probability $p_j = \tfrac{\tau_j(\bX)}{\sum_{i=1}^n \tau_{i}(\bX)}$. We generate $m$ i.i.d. random indices $j_1,\ldots,j_m\sim p$, each taking values in $[n]$. 
 
We next define a sampling-and-reweighting matrix $\bS$ to succinctly represent a subsampled regression problem, which will be used for active learning. Given indices $j_1,j_2,\ldots, j_m$ sampled from $p$, we construct $\bS \in \R^{m \times n}$, by setting the $i$-th row of $\bS$ to be $\frac{1}{\sqrt{mp_{j_i}}}{\bm e_{j_i}}$. We let $\hat\loss(f,\bw)$ denote: 
\begin{equation}
 \hat\loss(f,\bw) \coloneqq  \|\bS f(\bX\bw) - \bS\by\|^2 = \frac1m \sum_{i=1}^m \frac{1}{p_{j_i}} \abs{f(\langle \bw, \bx_{j_i}\rangle) - y_{j_i}}^2 .
 \label{eqn:subsampled-ls}
\end{equation}
One can verify that $\bbE[\bS^\top\bS]=\bI$ and hence $\hat\loss(f,\bw)$ is equal to $\loss(f,\bw)$ in expectation. Importantly, however, this subsampled loss can be evaluated using only those target values that appear in $\bS\by$, i.e. using at most $m$ entries from $\by$. So, our approach is to minimize $\hat\loss$ as a surrogate for $\loss$. 

\section{Main Results}\label{sec:main_results}
While it is tempting to try to obtain an $\eps$-accurate solution by returning any minimizer for the subsampled loss $\hat\loss$, and doing so works in the linear setting, it can be seen that such an approach will fail when $f$ is non-linear. In particular, consider the case when the rows of $\bX$ contain all $2^d$ binary vectors of length $d$, $f$ is a ReLU non-linearity, and $\by$ contains a $1$ in a single vector in $\{0,1\}^d$. Any sampling strategy that takes $o(2^d)$ samples will only observe labels equal to $0$. Thus, since $f(x) = \max(x, 0)$, \emph{any vector} $\hat{\bw}$ with non-positive entries is a valid minimizer for $\hat\loss$. However, by choose $\hat{\bw}$ to have very large negative entries, we can make $\loss(f,\hat{\bw})$ arbitrarily bad. 

\citet{gajjar2023active} deal with this issue by imposing a hard constraint on the norm of $\bw$. We instead consider 
minimizing a regularized version of $\hat\loss$ that penalizes $\bw$ with high norm. Ultimately, our regularization approach allows us to improve the $O(\eps^{-4}L^8)$ dependence in sample complexity in \citet{gajjar2023active} to  $O(\eps^{-2}L^4)$.
Concretely, we define the regularized loss as:
\begin{equation}\label{eq:reg_obj}
\hat\loss_{\reg}(f, \bw) \coloneqq \hat\loss(f, \bw) + \eps\|\bX\bw\|^2 = \norm{\bS f(\bX \bw) - \bS\by}^2 + \eps \norm{\bX \bw}^2.
\end{equation}

Our first main result is that, for a fixed $f$, the weight $\hat{\bw}$ that minimizes ~\eqref{eq:reg_obj}, is an $\eps$-accurate solution for  $\mathcal{L}(f,\bw)$ with high probability.

\begin{reptheorem}{thm:main_known_f} {
Let $\hat{\bw}$ solve the subsampled regularized least squares problem:
\begin{equation}
\label{eqn:ls-fixed-f}
\hat{\bw} \coloneqq \argmin_{
\bw\in\bbR^d} \hat\loss_{\reg}(f, \bw).
\end{equation}
Then, for some universal constant $c>0$, as long as $m\ge c L^4 \varepsilon^{-2} d\log^3\!d$, $(f, \hat\bw)$ is an $\eps$-accurate solution to of~\eqref{eqn:ls} with $\mathcal{F} = \{f\}$ with probability at least $99/100$.
}
\end{reptheorem}
The main ingredient for proving this theorem is to show that $\|\bS f(\bX\bw) - \bS f(\bX\bw^\star)\|^2$ is close to $ \|f(\bX\bw) - f(\bX\bw^\star)\|^2$ for all $\bw$ with sufficiently bounded norm, which is formally characterized by the following nonlinear subspace embedding lemma. 
\begin{lemma}[Non-linear subspace embedding with fixed non-linearity]\label{lem:fixed_concentration}
Let $\bS$ be a leverage score subsampling matrix with $m$ rows as defined in \Cref{sec:prelims}. Assume $\bX\in\bbR^{n\times d}$ has orthonormal columns.\footnote{As we will discuss later, $\bX$ can be assumed to be orthonormal without loss of generality.} There is some universal constant $C>0$ such that, for any $f\in\Lip_L$, and any $R>0$,
as long as $m\ge CL^4 \eps^{-2}d \log^3(d) \cdot \log(1/\delta)$, the following holds with probability $\ge 1-\delta$.
\begin{align*}
\left|{\norm{\bS f(\bX\bw_1) - \bS f(\bX\bw_2)}^2 - \norm{f(\bX\bw_1) - f(\bX\bw_2)}^2}\right| \le \eps R^2, 
\quad
\forall\bw_1, \bw_2\in B(R).
\end{align*}
\end{lemma}

This name ``nonlinear subspace embedding'' is a nod to the now standard subspace embedding result for leverage score sampling \citep{Sarlos:2006,DrineasMahoneyMuthukrishnan:2006}. We will use this result as well, which is a special (and stronger) case of Lemma~\ref{lem:fixed_concentration} when $f$ is the identity function.
\begin{lemma}[Subspace embedding]
\label{lem:subspace-embedding}
For any $\eps, \delta \in (0,1)$, as long as $m\gtrsim\frac{d \log(d/\delta)}{\eps^2}$, then with probability at least $1-\delta$, for all $\bw_1,\bw_2 \in \R^d$, 
\begin{equation}
\label{eqn:subspace-embedding}
(1-\eps)\|\bX\bw_1 - \bX\bw_2\|^2
\le \|\bS\bX\bw_1 - \bS\bX\bw_2\|^2 
\le (1+\eps)\|\bX\bw_1 - \bX\bw_2\|^2.
\end{equation}
\end{lemma}
Lemma~\ref{lem:subspace-embedding} can be proven using a matrix Chernoff bound (see, e.g., \cite{TCS-060}). Our \Cref{lem:fixed_concentration} requires more work due to the potential nonlinearity of $f$. Its proof is discussed in Section~\ref{sec:concentration_proof} and, along with a generalization to unknown $f$ in Lemma~\ref{lem:unknown_concentration}, is the major technical contribution of our work. Before getting into the details, we show how  \Cref{lem:fixed_concentration} can be used to prove \Cref{thm:main_known_f}.

\begingroup
\renewcommand{\proofname}{Proof of Theorem~\ref{thm:main_known_f}}
\begin{proof}
We simplify the discussion by assuming that, without loss of generality, $\bX$ has orthonormal columns. This can be done because leverage score are invariant under column transformations. If $\bX=\bQ\bR$ where $\bQ\in\bbR^{n\times d}$ has orthonormal columns and $\bR\in\bbR^{d\times d}$ is invertible, denoting by $\bq_j$ the $j$-th row of $\bQ$, one has $\bq_j = \bR^\top\bx_j$ and may verify
\begin{align*}
\tau_j(\bQ) = \bq_j^\top (\bQ^\top\bQ)^{-1} \bq_j = \bx_j^\top (\bX^\top \bX) \bx_j = \tau_j(\bX).
\end{align*}
The conclusion then follows from observing that all the involved statements are not affected if we substitute $\bX$ with $\bQ$ and $\bw$ with $\bR\bw$. The same approach was used in \cite{gajjar2023active}.

When $\bX$ is orthonormal, the leverage scores have a particularly simple form. 
\begin{equation}
\label{eqn:orth_score}
\tau_j(\bX) = \|\bx_j\|^2,
\quad \sum_{j=1}^n\tau_j(\bX) = d.
\end{equation}
Throughout the proof, we shall condition on the event that 
\begin{equation}
\label{eqn:Markov}
\hat\loss(f, \bw^\star)\le 1000\cdot\loss(f, \bw^\star)=1000\cdot \OPT,
\end{equation}
which happens with probability at least $0.999$ by Markov's inequality (since $\E\hat\loss=\loss$).
Now, since 
\[f(\bX \hat{\bw}) - \by = \big(f(\bX \hat{\bw}) - f(\bX\bw^\star)\big) + \big(f(\bX\bw^\star) - \by\big),\] 
from triangle inequality and the fact that $(a+b)^2 \leq 2a^2 + 2b^2$, we get the following 
\begin{align}
\label{eqn:fixed-tmp1}
\loss(f, \hat{\bw} ) = \norm{f(\bX \hat{\bw}) - \by}^2 &\le 2 \norm{f(\bX\hat{\bw}) - f(\bX \bw^\star) }^2 + 2 \OPT.
\end{align}
Next, by optimality of $\hat{\bw}$, we have $\hat\loss_{\reg}(f, \hat{\bw})\le \hat\loss_{\reg}(f, \bw^\star)$, i.e., 
\begin{equation}
\label{eqn:fixed_optimality}
\|\bS f(\bX\hat{\bw}) - \bS \by\|^2 + \eps \|\bX\hat{\bw}\|^2 \le \hat\loss(f, \bw^\star) + \eps \|\bX\bw^\star\|^2.
\end{equation}
In particular, we have (recalling that $\bX$ is assumed to have orthonormal columns)
\[
\|\hat{\bw}\|^2 = \|\bX\hat{\bw}\|^2 
\le \frac{1}{\eps}\hat\loss(f, \bw^\star) + \|\bX\bw^\star\|^2 
\le \frac{1000}{\eps}\OPT + \|\bX\bw^\star\|^2.
\]
This implies that $\hat{\bw}\in B(R)$ where $R=\sqrt{\frac{1000}{\eps}\OPT+\|\bX\bw^\star\|^2}$. Then, from \Cref{lem:fixed_concentration}, the following holds with probability at least $0.999$ given $m\ge CL^4\eps^{-2}d\log^3\!d$:
\begin{align*}
\|f(\bX\hat{\bw}) - f(\bX\bw^\star)\|^2
& \le \|\bS f(\bX\hat{\bw}) - \bS f(\bX\bw^\star)\|^2 + \eps R^2 \\
& \lesssim \|\bS f(\bX\hat{\bw}) - \bS\by\|^2 + \|\bS f(\bX\bw^\star) - \bS\by\|^2 + \eps R^2 \\
& \lesssim \left(\hat\loss(f, \bw^\star) + \eps\|\bX\bw^\star\|^2\right) + \hat\loss(f, \bw^\star) + \eps\left(\frac{1}{\eps}\OPT + \|\bX\bw^\star\|^2\right) \\
& \lesssim \OPT + \eps\|\bX\bw^\star\|^2,
\end{align*}
where the second line is triangle inequality, the third line used \eqref{eqn:fixed_optimality} and the definition of $R$, and the last line used \eqref{eqn:Markov}.
Plugging the above inequality into \eqref{eqn:fixed-tmp1} readily gives
\[
\loss(f, \hat{\bw}) \lesssim \OPT + \eps\|\bX\bw^\star\|^2,
\]
as desired (after replacing $\eps$ by $\eps/C$ for some sufficiently large constant $C>0$).
\end{proof}
\endgroup

Our second main result is for the setting where $f$ is an unknown $L$-Lipschitz function vanishing at $0$. 
Again, we show that an $\eps$-accurate solution to the problem of $\min_{f, \bw} \mathcal{L}(f,\bw)$ can be obtained by minimizing the  regularized subsampled loss, but this time over both $f$ and $\bw$. Formally, we have:
\begin{reptheorem}{thm:unknown_f}
Let $\hat f$ and $\hat{\bw}$ be the solution to the subsampled regularized least square problem:
\begin{equation}
\label{eqn:ls-unknown-f}
(\hat f, \hat{\bw}) = \argmin_{\substack{f\in\Lip_L\\
\bw\in\bbR^d}} \hat\loss_{\reg}(f,\bw).
\end{equation}
Then, for some universal constant $C>0$, as long as $m\ge C L^4 \varepsilon^{-2} d(\log^2 n + \log^3\!d)$, $(\hat f, \hat\bw)$ is an $\varepsilon$-accurate solution to~\eqref{eqn:ls} with $\mathcal{F} = \Lip_L$ probability at least $99/100$.
\end{reptheorem}
Similar to Theorem~\ref{thm:main_known_f}, this result crucially depends on the following nonlinear concentration result, which extends Lemma~\ref{lem:fixed_concentration} to the case when $\bX\bw_1$ and $\bX\bw_2$ are transformed by two potentially different $L$-Lipschitz functions, $f_1$ and $f_2$.

\begin{lemma}[Non-linear subspace embedding with unknown non-linearity]\label{lem:unknown_concentration}
Let $\bS$ be a leverage score subsampling matrix with $m$ rows as defined in \Cref{sec:prelims}. Assume $\bX\in\bbR^{n\times d}$ has orthonormal columns.
As long as $m\ge C L^4 \varepsilon^{-2} d(\log^2 n + \log^3 (d/\delta))$ for some fixed constant $C > 0$, the following holds with probability $\ge 1-\delta$.
\begin{align*}
    \abs{\norm{\bS f_1(\bX\bw_1) - \bS f_2(\bX\bw_2)}^2 - \norm{f_1(\bX\bw_1) - f_2(\bX\bw_2)}^2} \le \eps R^2, 
\end{align*}
for all $f_1,f_2\in\Lip_L$ and for all $\bw_1, \bw_2\in B(R).$
\end{lemma}

The proof of Theorem~\ref{thm:unknown_f} follows the same steps as that of Theorem~\ref{thm:main_known_f}, with the distinction that we must compare $\loss (\hat f, \hat \bw)$ with the optimal loss $\loss (f^\star, \bw^\star)$. 
\begin{proof}(Proof of Theorem~\ref{thm:unknown_f})
We simplify the discussion by assuming that, without loss of generality, $\bX$ has orthonormal columns, just like Theorem~\ref{thm:main_known_f}. 
Throughout the proof, we shall condition on the event that 
\begin{equation}
\hat\loss(f^\star, \bw^\star)\le 1000\cdot\loss(f^\star, \bw^\star)=1000\cdot \OPT,
\end{equation}
which happens with probability at least $0.999$ by Markov's inequality (since $\E\hat\loss=\loss$).
Now, since 
\[\hat f(\bX \hat{\bw}) - \by = \big(\hat f(\bX \hat{\bw}) - f^\star(\bX\bw^\star)\big) + \big(f^\star(\bX\bw^\star) - \by\big),\] 
from triangle inequality and the fact that $(a+b)^2 \leq 2a^2 + 2b^2$, we get the following 
\begin{align}
\label{eqn:diff-tmp1}
\loss(\hat f, \hat{\bw} ) = \norm{\hat f(\bX \hat{\bw}) - \by}^2 &\le 2 \norm{\hat f(\bX\hat{\bw}) - f^\star(\bX \bw^\star) }^2 + 2 \OPT.
\end{align}
Next, by optimality of $\hat f,\hat{\bw}$, we have $\hat\loss_{\reg}(\hat f, \hat{\bw})\le \hat\loss_{\reg}(f^\star, \bw^\star)$, i.e., 
\begin{equation}
\label{eqn:diff_optimality}
\|\bS \hat f(\bX\hat{\bw}) - \bS \by\|^2 + \eps \|\bX\hat{\bw}\|^2 \le \hat\loss(f^\star, \bw^\star) + \eps \|\bX\bw^\star\|^2.
\end{equation}
In particular, we have (recalling that $\bX$ is assumed to have orthonormal columns)
\[
\|\hat{\bw}\|^2 = \|\bX\hat{\bw}\|^2 
\le \frac{1}{\eps}\hat\loss(f^\star, \bw^\star) + \|\bX\bw^\star\|^2 
\le \frac{1000}{\eps}\OPT + \|\bX\bw^\star\|^2.
\]
This implies that $\hat{\bw}\in B(R)$ where $R=\sqrt{\frac{1000}{\eps}\OPT+\|\bX\bw^\star\|^2}$. Then, from \Cref{lem:unknown_concentration}, the following holds with probability at least $0.999$ given $m\ge CL^4\eps^{-2}d\log^3\!d$:
\begin{align*}
\|\hat f(\bX\hat{\bw}) - f^\star(\bX\bw^\star)\|^2
& \le \|\bS \hat f(\bX\hat{\bw}) - \bS f^\star(\bX\bw^\star)\|^2 + \eps R^2 \\
& \lesssim \|\bS \hat f(\bX\hat{\bw}) - \bS\by\|^2 + \|\bS f^\star(\bX\bw^\star) - \bS\by\|^2 + \eps R^2 \\
& \lesssim \|\bS \hat f(\bX\hat{\bw}) - \bS\by\|^2 + \|\bS f^\star(\bX\bw^\star) - \bS\by\|^2 + \eps R^2 \\
& \lesssim \left(\hat\loss(f^\star, \bw^\star) + \eps\|\bX\bw^\star\|^2\right) + \hat\loss(f^\star, \bw^\star) + \eps\left(\frac{1}{\eps}\OPT + \|\bX\bw^\star\|^2\right) \\
& \lesssim \OPT + \eps\|\bX\bw^\star\|^2,
\end{align*}
where the second line is triangle inequality, the third line used \eqref{eqn:diff_optimality} and the definition of $R$, and the last line used \eqref{eqn:Markov}.
Plugging the above inequality into \eqref{eqn:diff-tmp1} readily gives
\[
\loss(\hat f, \hat{\bw}) \lesssim \OPT + \eps\|\bX\bw^\star\|^2,
\]
as desired (after replacing $\eps$ by $\eps/C$ for some sufficiently large constant $C>0$).
\end{proof}

\section{Proof of the non-linear concentrations}\label{sec:concentration_proof}
The proofs of Lemma~\ref{lem:fixed_concentration} and Lemma~\ref{lem:unknown_concentration}, especially the latter one, are the major challenges in this work. We sketch the key ideas below. 

\paragraph{Proof roadmap for Lemma~\ref{lem:fixed_concentration}} 
Our proof (given in Appendix~\ref{subsec:apd:known}) is based on ideas from \citet{rudelson1996random}. However, we employ a simplified version, using dual Sudakov minoration, deviating from of the construction of chaining functionals therein. This simplification significantly streamlines the explanation, although at the cost of losing a $\log^2 d$ factor.

\paragraph{Step 1: Symmetrization.} We want to obtain a tail bound for the supremum of our random process.
     For this, we apply a standard symmetrization technique (Lemma \ref{lem:symmetrization}), where we introduce i.i.d. Rademacher random variables $\xi_1,\xi_2,\ldots, \xi_m$, where each $\xi_i$ takes values in $-1$ and $1$ with probabilities $1/2$ each.
     This reduces the problem to finding tail bounds of the following Bernoulli process:
    \[Z(\bw_1, \bw_2)\coloneqq \sum_{i=1}^m \xi_i \frac{\big(\fxw{j_i}{\bw_1} - \fxw{j_i}{\bw_2} \big)^2}{p_{j_i}},
\qquad (\bw_1,\bw_2)\in B(R)\times B(R).
\]
In particular, this allows us to prove Lemma~\ref{lem:fixed_concentration} by showing $\frac1m \sup_{\bw_1,\bw_2\in B(R)} |Z(\bw_1,\bw_2)| \lesssim \eps$ with overwhelming probability.
\paragraph{Step 2: Applying sub-Gaussianity of symmetrized process.} It turns out that, conditioned on the samples $\{j_1,j_2,\ldots, j_m\}$ (fixing $\bS)$, the above random process is sub-Gaussian w.r.t. the index $(\bw_1,\bw_2)$, 
endowed with the following metric on $B(R) \times B(R)$ (Appendix~\ref{subsec:apd:symmetrization}).
\begin{align*}
&\rho_f\big( (\bw_1, \bw_2), (\bw_1', \bw_2') \big)
\\
&\quad\coloneqq
\left( \sum_{i=1}^m \frac{1}{p_{j_i}^2}{\big(f(\langle \bx_{j_i}, \bw_1\rangle) - f(\langle \bx_{j_i}, \bw_2\rangle) \big)^2 - \big(f(\langle \bx_{j_i}, \bw_1'\rangle) - f(\langle \bx_{j_i}, \bw_2'\rangle) \big)^2} \right)^{1/2}.
\end{align*}

\paragraph{Step 3: Bounding the sub-Gaussian norm by the dual norm of a polytope.} 

The Lipschitz continuity of $f$ allows us to obtain an upper bound (Lemma \ref{lem:fixed-metric-bound}) of $\rho_f$ in terms of the dual norm of  the polytope $\mathcal P$ formed by vertices $\{\pm\bx_{j_i}/\sqrt{p_{j_i}}\}_{i\in[m]}$:
\[
\rho_f\big( (\bw_1, \bw_2), (\bw_1', \bw_2') \big) \lesssim L^2 R\sqrt{m}\left(\|\bw_1 - \bw_1'\|_{\mathcal P^\circ} + \|\bw_2 - \bw_2'\|_{\mathcal P^\circ}\right), \forall \bw_1,\bw_2,\bw_1',\bw_2'\in B(R)
\]
with high probability, where $\mathcal P^\circ$ is the polar of $\mathcal P$ and $\|\cdot\|_{\mathcal P^\circ}$ is its Minkowski norm. This bound gets rid of any dependence on the non-linearity $f$ and prepares us for the application of dual Sudakov minoration (Lemma \ref{lem:sudakov}).
\paragraph{Step 4: Bound in expectation via Dudley's integral and dual Sudakov minoration.} The upper bound of $\rho_f$, and the sub-Gaussianity of $Z(\bw_1,\bw_2)$ allows us to obtain upper bound for the expectation of the random process using Dudley's integral with respect to  $\|\cdot\|_{\mathcal P^\circ}$ (Appendix ~\ref{subsec:apd:fixed_expectation}). This, in turn, is controlled using dual Sudakov minoration, leading to
\[\bbE \sup_{\bw_1,\bw_2} |Z(\bw_1, \bw_2)|
\lesssim L^2 R^2 \sqrt{md\log^2d\cdot\log m}.\]
When $m\gtrsim \eps^{-2}d\log^3 d$, this implies $\frac1m \bbE \sup_{\bw_1,\bw_2\in B(R)} |Z(\bw_1,\bw_2)| \lesssim \eps$, which is already close to what we desire.
\paragraph{Step 5: Tail bound via concentration of measure. }Finally, we deduce the desired tail bound from the above bound in expectation using a concentration of measure argument in Appendix~\ref{subsec:apd:fixed_tail}.

\paragraph{Proof roadmap for Lemma~\ref{lem:unknown_concentration}.} 
As discussed, a primary challenge in dealing with the unknown $f$ is that, to prove the bound, we must construct a discretization for the infinite class of functions $\Lip_L$. We require an efficient discretization of $\Lip_L$ that avoids any polynomial dependencies on $n$. We outline our approach below, and the complete proofs are provided in Appendix \ref{subsec:apd:unknown}. \\

\paragraph{Step 1: Symmetrization.} Similar to Lemma \ref{lem:fixed_concentration}, the problem can be reduced to obtaining tail bounds for the supremum of the symmetrized process, but now considering two different functions $f_1$ and $f_2$:  
    \[
Z_{f_1,f_2}(\bw_1, \bw_2)\coloneqq \sum_{i=1}^m \xi_i \frac{\big( f_1(\langle \bx_{j_i}, \bw_1\rangle) - f_2(\langle\bx_{j_i}, \bw_2\rangle \big)^2}{p_{j_i}},
\]
where $f_1,f_2\in\Lip_L,~(\bw_1,\bw_2)\in B(R)\times B(R)$.\\

\paragraph{Step 2: Applying generic chaining to find guidelines for discretization.} The next step involves discretizing $\Lip_L$ for a ``good'' approximation, while maintaining a small ``size''. Using generic chaining for Bernoulli process, we construct (in Appendix~\ref{subsec:apd:pf_chaining}) two ``sampling-aware'' metrics $D_\infty$ and $D_2$ (depending on $\bS$) on  $\Lip_L\times \Lip_L$ measuring respectively the ``goodness'' of approximation and the size of the discretization. We show (in Lemma~\ref{lem:chaining}) that if $\Net_\Delta \subset \Lip_L$ is such that $\Net_\Delta \times \Net_\Delta$ is a $\Delta$-net of $\Lip_L\times\Lip_L$ in $D_\infty$ distance, then\footnote{Here $\Delta$ can be a random variable depending on $\bS$, as $D_\infty$ is already a random distance depending on $\bS$. This accounts for the expectation in $\bbE\Delta$.}
\begin{align*}
& \E \sup_{\substack{f_1,f_2\in\Lip_L \\ \bw_1, \bw_2\in B(R)}} \left|Z_{f_1, f_2}(\bw_1,\bw_2) \right| \lesssim L^2 R^2\sqrt{md\log^2\!d \cdot \log m} + \E \Delta + \E  \int_0^\infty\sqrt{\log \Net(\Net_\Delta\times \Net_\Delta, D_2, \eps)}~\mathrm{d}\eps.
\end{align*}


\paragraph{Step 3: Net construction.} The construction of the desired net, carried out in Appendix~\ref{subsec:apd:pf-D-infty-net}, is the major technical challenge in the proof. 
Our approach is guided by the following upper bound on $D_\infty$, proved in Lemma~\ref{lem:D_bounds}:
\[D_\infty\big((f_1, f_2), (f_1', f_2')\big) \lesssim \rho_{\infty}(f_1, f_1') + \rho_{\infty}(f_2, f_2'),
\]
where $\rho_\infty$ is some metric on $\Lip_L$ obeying (with high probability)
\[ \rho_\infty(f_1, f_2) \lesssim LR\sqrt{m}
\left( \sum_{i=1}^m \frac{1}{\|\bx_{j_i}\|^2} \|f_1 - f_2\|^2_{L^\infty([-R\|\bx_{j_i}\|, R\|\bx_{j_i}\|])} \right)^{1/2}.
\]  

We then construct a net $\Net_\Delta$ (details are in Appendix~\ref{subsubsec:apd:net_construction}) by imposing that for any $f\in\Lip_L$, there is some $\tilde f\in\Net_\Delta$, such that $|\tilde f(x) - f(x)|\lesssim \eps|x|$ for larger values of $|x|$, while $|\tilde f(x) - f(x)|\lesssim \eps$ for smaller values of $|x|$. The former restriction ensures $\frac{1}{\|\bx_{j_i}\|^2}\|f_1 - f_2\|^2_{L^\infty([-R\|\bx_{j_i}\|, R\|\bx_{j_i}\|])}\lesssim \eps^2$ if $\|\bx_{j_i}\|$ is large, but will require a net of infinite cardinality if demanded to hold for all $x$. The latter restriction is a careful relaxation of the former, which, as we shall show in the proof, does not hurt the quality of approximation by the net (up to logarithmic factors). 


\paragraph{Step 4: Bounding the size of the discretization via embedding to simple spaces.} Subsequently, to bound the Dudley's integral of $\Net_\Delta\times\Net_\Delta$ in $D_2$ distance (see Appendix~\ref{subsubsec:apd:dudley_bound}), we introduce another technique derived from the proof of Lemma~\ref{lem:D_bounds} and the explicit construction of $\Net_\Delta$. In essence, we construct an embedding 
(a combination of Lemma~\ref{lem:D_bounds}, Lemma~\ref{lem:net_embed_discrete} and Lemma~\ref{lem:discrete_embed_lip}) of $(\Net_\Delta\times\Net_\Delta, D_2)$ into $(\Lip_1, L^\infty([-1, 1]))^2$ (within an additive error which is well controlled). 
Since the Dudley's integral of $(\Lip_1, L^\infty([-1, 1])$ has a well-known upper bound (Lemma~\ref{lem:lip_entropy}), this embedding would imply our desired bound easily. 
\section{Conclusion}\label{sec:conclusion}
This work provides the first sample complexity results for actively learning single index models in the agnostic setting with a nearly linear dependence on the dimension $d$, and with no strong distributional assumptions. We believe the results suggest a number of avenues for future exploration. For example, while we obtain near optimal sample complexity results, we do not consider computational efficiency, which has been considered in work on fitting single index functions in related models. Additionally, to the best of our knowledge, the related (and harder) multi-index model has not been addressed in the same setting that we consider. Finally, although our \Cref{thm:main_known_f,thm:unknown_f} provide a constant-factor multiplicative approximation, we believe it should be possible to obtain a $(1+\eps)$ approximation using techniques from \citet{MuscoMuscoWoodruff:2022}.

\subsubsection*{Acknowledgements}
Christopher Musco is supported by the NSF Career Award (2045590) and the United States Department of Energy grant (DE-SC0022266). Chinmay Hegde is supported by the NSF award (CCF-2005804). Yi Li is supported in part by the Singapore Ministry of Education (AcRF) Tier 2 grant (MOE-T2EP20122-0001) and Tier 1 grant (RG75/21). Wai Ming Tai is supported by the Singapore Ministry of Education (AcRF) Tier 2 grant (MOE-T2EP20122-0001).  
Xingyu Xu is supported by the NSF grant (ECCS-2126634) and the Axel Berny Presidential Graduate Fellowship at Carnegie Mellon University.  

\bibliographystyle{plainnat}
\bibliography{ref}
\clearpage
\appendix

\section{Tools from High Dimensional Probability}
In this section, we introduce the tools which are used for our main results. 
\subsection{Covering numbers}
We shall need two slightly different definitions of covering numbers, which will coincide in most useful cases.
\paragraph{Covering number with respect to a set.} Let $K, T$ be subsets in $\bbR^n$. The covering number (with respect to set $T$) $\Net(K, T)$ is defined as
\[
  \Net(K, T)\coloneqq \min \left\{n\in\mathbb{Z}, n\ge 0 : \exists x_1,\ldots, x_n\in K, \text{ such that } K\subset\cup_{i=1}^n (x_i+T) \right\}.
\]
The above can also be seen as the minimum number of translations over $T$ required to cover $K$.
\paragraph{Covering number with respect to a metric.} Let $K\subset \bbR^n$. The covering number (with respect to metric $d$) $\Net(K, d, \eps)$ is defined as 
\[
  \Net(K, d, \eps) \coloneqq \min \left\{n\in\mathbb{Z}, n\ge 0 : \exists x_1,\ldots, x_n\in K, \text{ such that } \min_{i\le n} d(x, x_i)\le \eps, \forall x\in K \right\}.
\]
When $d$ is the Euclidean distance, we omit $d$ and simply denote by $\Net(K,\eps)$ the corresponding covering number. 

\paragraph{When two definitions agree.} 
If $T$ is the unit ball of some norm $\|\cdot\|_T$ on $\bbR^n$, and $d_T$ is the distance induced by the norm $\|\cdot\|_T$, then the above two definitions agree:
\[
\Net(K, \eps T) = \Net(K, d_T, \eps).
\]

\paragraph{Volumetric estimate. }
The following result is standard (or follows from the well-known argument based on packing number and volume), c.f. Lemma~9.5 in \cite{ledoux1991probability}.
\begin{lemma}
Let $T$ be a symmetric convex body in $\bbR^n$ with positive volume.
We have, for all $0<\eps<1$, that
\[
\log\Net(T, \eps T) \lesssim n\log(2/\eps).
\]
\end{lemma}

Combined with the chain inequality $\Net(K_1, K_3)\le \Net(K_1, K_2)\Net(K_2, K_3)$, this implies
\begin{lemma}
\label{lem:volume}
Let $K$ be a subset of $\bbR^n$ and $T$ be a symmetric convex body in $\bbR^n$ with positive volume.
We have, for all $0<\eps<1$, that
\[
\log\Net(K, \eps T) \lesssim n\log(2/\eps) + \log\Net(K, T).
\]
\end{lemma}

\subsection{Dudley's inequality}
Let $(T,d)$ be a (pseudo-)metric space. Let $(X_t)_{t\in T}$ be a random process on $T$ with zero mean, i.e., $\E X_t=0$ for all $t\in T$. The process $X_t$ is said to be subgaussian (with respect to $d$) if 
\[
\bbP(|X_s-X_t| > t) \le 2\exp\left(-\frac{t^2}{2d^2(s, t)}\right),\quad\forall s,t\in T.
\]
\begin{lemma}[Subgaussianity of Bernoulli process, \cite{vershynin_2018}]
Let $T\subset \bbR^n$ and $\xi_1,\ldots,\xi_n$ be i.i.d. Rademacher random variables. Then the process
\[X_t\coloneqq \sum_{i=1}^n \xi_i t_i, \quad t=(t_1,\ldots,t_n)\in T\]
is subgaussian on $T$ with respect to the Euclidean distance.
\end{lemma}

\begin{lemma}[Dudley's inequality, \cite{vershynin_2018}]
\label{lem:dudley}
If $(X_t)_{t\in T}$ is a subgaussian process with respect to a metric $d$, then
\begin{equation*}
    \E\sup_{t\in T}|X_t| \lesssim  \inf_{t\in T}\E|X_t| + \int_0^\infty \sqrt{\log(\Net(T, d, \eps))}~\mathrm{d}\eps.
\end{equation*}

\end{lemma}

It is well-known that Dudley's inequality can be refined to a tail bound.
\begin{lemma}[Dudley's inequality, tail bound, \cite{talagrand2022upper}]
\label{lem:dudley-tail}
If $(X_t)_{t\in T}$ is a subgaussian process with respect to a metric $d$, then the following holds for any $t>0$ with probability at least $1-2\exp(-t^2/2)$:
\begin{equation*}
\sup_{t\in T}|X_t| \lesssim   \inf_{t\in T}\E|X_t| + \int_0^\infty \sqrt{\log (\Net(T, d, \eps))}~\mathrm{d}\eps
+ t\cdot\Diam(T, d).
\end{equation*}
\end{lemma}

The following upper bound for subgaussian processes supported on a finite set is well-known and is a simple corollary of Dudley's inequality.
\begin{corollary}\label{cor:dudley-finite}
If $(X_t)_{t\in T}$ is a subgaussian process on a finite set $T$ with respect to a pseudo-metric $d$. 
Then
\[
\sup_{t\in T}|X_t| \lesssim   \inf_{t\in T}\E|X_t| + \Diam(T, d)\sqrt{\log|T|}.
\]
In particular, if $T\subset\bbR^n$ and $d$ is the Euclidean distance, then
\[
\sup_{t\in T}|X_t|\lesssim \sup_{t\in T}\|t\|\cdot\sqrt{\log(|T|+1)}.
\]
\end{corollary}
\begin{proof}
For the first part, note that $\Net(T, d, \eps)\le|T|$ for any $\eps>0$, and moreover $\Net(T, d, \eps) = 1$ for $\eps>\Diam(T, d)$. The conclusion the follows from Lemma~\ref{lem:dudley}.

For the second part, we apply the conclusion of the first part to the set $T\cup\{0\}$ and note that 
\[ \Diam(T\cup\{0\})=\sup_{t_1,t_2\in T\cup\{0\}}\|t_1-t_2\|\le \sup_{t_1,t_2\in T\cup\{0\}}(\|t_1\|+\|t_2\|)
\le 2\sup_{t\in T}\|t\|.
\]
\end{proof}

Dudley's inequality has a partial inverse, which can be viewed as a useful probabilistic estimate of covering numbers.
\begin{lemma}[Dual Sudakov minoration, \cite{ledoux1991probability}]
\label{lem:sudakov}
Let $B$ be the unit (Euclidean) ball in of $\bbR^n$, $T$ be a symmetric convex body in $\bbR^n$, and $g$ be the standard $n$-dimensional random normal vector. 
Denote by $T^\circ\coloneqq\{x: \langle x, y\rangle\le 1, \forall y\in T\}$ the polar of $T$.
Then
\[
\sup_{\eps>0} \eps\sqrt{\log \Net(B,\eps T^\circ)} \lesssim \E\sup_{t\in T}\langle t, g\rangle.
\]
\end{lemma}

Dudley's integral is subadditive with respect to the metric $d$ and tensorizes well, two basic properties that will greatly simplify our proof. 
\begin{lemma}
[Sublinearity]
\label{lem:sublinear-gamma2}
Let $d_1$, $d_2$ be pseudo-metrics on some set $T$. Let $a,b>0$ be real numbers. Then
\[
\int_0^\infty\sqrt{\log\Net(T, a d_1 + b d_2, \eps)}~\mathrm{d}\eps \lesssim a \int_0^\infty\sqrt{\log\Net(T, d_1, \eps)}~\mathrm{d}\eps + b \int_0^\infty\sqrt{\log\Net(T, d_2, \eps)}~\mathrm{d}\eps.
\]
\end{lemma}

\begin{lemma}
[Tensorization]
\label{lem:product-gamma2}
Let $(T_1, d_1)$, $(T_2, d_2)$ be pseudo-metric spaces. Then $T_1\times T_2$ can be endowed with a natural pseudo-metric, defined by
\[
d\big((t_1, t_2), (t_1', t_2')\big) 
= d_1(t_1, t_1') + d_2(t_2, t_2').
\]

We then have
\[
\int_0^\infty\sqrt{\log\Net(T_1\times T_2, d, \eps)}~\mathrm{d}\eps \lesssim \int_0^\infty\sqrt{\log\Net(T_1, d_1, \eps)}~\mathrm{d}\eps + \int_0^\infty\sqrt{\log\Net(T_2, d_2, \eps)}~\mathrm{d}\eps.
\]
\end{lemma}

\section{Proof of subspace embedding with fixed nonlinearity, Lemma \ref{lem:fixed_concentration}}
\label{subsec:apd:known}

The proof relies crucially on the idea developed in \cite{rudelson1996random}. We extend a simplified version of their idea into the Lipschitz nonlinear setting. First, we apply a standard symmetrization argument to the $\ell$th moment of the desired quantity, which is actually the supremum of deviation of $\|\bS f(\bX\bw_1) - \bS f(\bX\bw_2)\|^2$ over $\bw_1,\bw_2\in B(R)$. This reduces the problem to finding tail bounds of a Bernoulli process. Then, we utilize the sub-Gaussian property of that Bernoulli process and give an upper bound of its associated sub-Gaussian distances, which reduces the sub-Gaussian metric to the dual norm of some polytope determined by the sampling process. This allows us to invoke dual Sudakov minoration to bound the Dudley integral of the sub-Gaussian distance, ultimately leading to a tail bound of the Bernoulli process as we desire.

\subsection{Symmetrization}
To simplify notations, for the rest of the proof we denote \[v_{j_i}(\bw_1, \bw_2)\coloneqq \fxw{j_i}{\bw_1} - \fxw{j_i}{\bw_2}.\] 
\begin{lemma}[Symmetrization]\label{lem:symmetrization}
Let $\xi_1,\xi_2,\ldots, \xi_m$ be i.i.d. Rademacher random variables (independent of $\bS$), i.e., each $\xi_i$ takes values $-1$ and $1$ with probabilities $1/2$ each independently. Then
\begin{align*}
& \phantom{=} \bbE \sup_{\bw_1,\bw_2 \in B(R)} \left \lvert \norm{\bS f(\bX\bw_1) - \bS f(\bX \bw_2)}^2 - \norm{f(\bX\bw_1) - f(\bX \bw_2)}^2\right \rvert^\ell \\
& \le 2^\ell \cdot \bbE\sup_{\bw_1,\bw_2\in B(R)} \left| \frac{1}{m} \sum_{i=1}^m \xi_i \frac{(v_{j_i}(\bw_1,\bw_2))^2}{p_{j_i}} \right|^\ell.
\end{align*}
\end{lemma}
This result is standard. For sake of completeness, we provide a proof in Appendix~\ref{subsec:apd:symmetrization}.

\subsection{Bounding the Bernoulli process: the sub-Gaussian distance}
Lemma~\ref{lem:symmetrization} leads us to study the following symmetrized random process. 
\[Z(\bw_1, \bw_2)\coloneqq \sum_{i=1}^m \xi_i \frac{\big(v_{j_i}(\bw_1, \bw_2) \big)^2}{p_{j_i}},
\qquad (\bw_1,\bw_2)\in B(R)\times B(R),
\]


It is clear that the above random process conditioned on the samples $\{j_1,j_2,\ldots, j_m\}$ (fixing $\bS$) is a sub-Gaussian process (\cite{vershynin_2018}) with respect to index $(\bw_1,\bw_2)$ endowed with the metric $\rho_f$ (called the sub-Gaussian distance) on $B(R) \times B(R)$ defined as follows. 
\begin{equation}
\label{eqn:def_rho}
\rho_f\big((\bw_1,\bw_2), (\bw_1',\bw_2')\big)
\coloneqq \left( \sum_{i=1}^m \frac{1}{p_{j_i}^2}\big (v_{j_i}(\bw_1, \bw_2)^2 - v_{j_i}(\bw_1', \bw_2')^2 \big)^2 \right)^{1/2}.
\end{equation}

Based on sub-Gaussianity of $Z(\bw_1, \bw_2)$, we shall derive a near-optimal tail bound of it using Dudley's inequality (Lemma~\ref{lem:dudley}). This would entail establishing an easier-to-manipulate upper bound of the metric $\rho_f$. Before doing so (in Lemma~\ref{lem:fixed-metric-bound}) we need to set up a few notations.

Recall that we are conditioning on $\{j_1,j_2,\cdots,j_m\}$ in this part of the proof. We define a symmetric convex body
\[\P = \operatorname{conv}\left(\pm \frac{\bx_{j_1}}{\sqrt{p_{j_1}}},\pm \frac{\bx_{j_2}}{\sqrt{p_{j_2}}},\ldots \pm \frac{\bx_{j_m}}{\sqrt{p_{j_m}}}\right).\] 
We further denote $\P^\circ$ as its polar given by 
\[\P^\circ = \{\bz\in\bbR^d: \abs{\langle \bz,{\bm p}\rangle} \le 1, \forall {\bm p} \in \P\},\]
and denote $\|\cdot\|_{\mathcal P^\circ}$ as the Minkowski norm associated with $\P^\circ$, defined for $\bw \in \R^d$ as 
\[\norm{\bw}_{\P^\circ} \coloneqq \inf\{t>0: \bw\in t\P^\circ\}=\sup_{i\in[m]}\left|\left\langle \bx_{j_i}/\sqrt{p_{j_i}}, \bw\right\rangle\right|.\]

Using Lipschitz continuity of $f$, we can establish a useful upper bound of $\rho_f$ in terms of $\|\cdot\|_{\mathcal P^\circ}$, reducing the study of the sub-Gaussian distance to the dual norm associated to a certain polytope.
\begin{lemma}[Bounding the sub-Gaussian distance $\rho_f$]\label{lem:fixed-metric-bound}
With the above notations, one has for $f\in\Lip_L$ that
\[
\rho_f\le \rho,
\]
where $\rho$ is the metric on $B(R)\times B(R)$ defined by
\begin{align}
\rho\big((\bw_1,\bw_2), (\bw_1',\bw_2')\big)
= 4L^2 R\left\|\sum_{i=1}^m \frac{1}{p_{j_i}} \bx_{j_i}\bx_{j_i}^\top\right\|^{1/2}\left(\|\bw_1 - \bw_1'\|_{\mathcal P^\circ} + \|\bw_2 - \bw_2'\|_{\mathcal P^\circ}\right).
\label{eqn:rho}
\end{align}
\end{lemma}
\begin{proof}
The definition of $\rho_f$ and $v_{j_i}$ leads to the following.
\begin{align*}
| v_{j_i}^2(\bw_1, \bw_2) - v_{j_i}^2(\bw_1', \bw_2') |
&= \left| \left( f(\langle \bx_{j_i}, \bw_1\rangle) - f(\langle \bx_{j_i}, \bw_2\rangle)\right) - \left( f(\langle \bx_{j_i}, \bw_1'\rangle) - f(\langle \bx_{j_i}, \bw_2'\rangle) \right) \right| \\
&\phantom{=} \cdot \left| \left(f(\langle \bx_{j_i}, \bw_1\rangle) - f(\langle \bx_{j_i}, \bw_2\rangle)\right) + \left( f(\langle \bx_{j_i}, \bw_1'\rangle) - f(\langle \bx_{j_i}, \bw_2'\rangle)\right)  \right| 
\\
&\le L^2\left(\left|\langle \bx_{j_i}, \bw_1-\bw_1'\rangle\right| + \left|\langle \bx_{j_i}, \bw_2-\bw_2'\rangle\right|\right)\\
&\phantom{=} \cdot \left(|\langle \bx_{j_i}, \bw_1\rangle| + |\langle \bx_{j_i}, \bw_2\rangle| + |\langle \bx_{j_i}, \bw_1'\rangle| + |\langle \bx_{j_i}, \bw_2'\rangle|\right),
\end{align*}
where the inequality follows from $f$ being $L$-Lipschitz. 
Next we observe that
\begin{align*}
\left|\langle \bx_{j_i}, \bw_1-\bw_1'\rangle\right| + \left|\langle \bx_{j_i}, \bw_2-\bw_2'\rangle\right| & = \sqrt{p_{j_i}} \big( \left|\langle \bx_{j_i}/\sqrt{p_{j_i}}, \bw_1-\bw_1'\rangle\right| + \left|\langle \bx_{j_i}/\sqrt{p_{j_i}}, \bw_2-\bw_2'\rangle\right| \big)
\\
& \le \sqrt{p_{j_i}}(\|\bw_1-\bw_1'\|_{\mathcal P^\circ} + \|\bw_2-\bw_2'\|_{\mathcal P^\circ}),
\end{align*}
where the inequality follows from the definition of $\|\cdot\|_{\mathcal P^\circ}$. Combining the above two inequalities, one has
\begin{align}
&\rho_f^2\big( (\bw_1,\bw_2), (\bw_1', \bw_2') \big) \nonumber\\
& \le \sum_{i=1}^m \frac{1}{p_{j_i}^2}\cdot L^4 p_{j_i}(\|\bw_1-\bw_1'\|_{\mathcal P^\circ} + \|\bw_2-\bw_2'\|_{\mathcal P^\circ})^2 \cdot \left(|\langle \bx_{j_i}, \bw_1\rangle| + |\langle \bx_{j_i}, \bw_2\rangle| + |\langle \bx_{j_i}, \bw_1'\rangle| + |\langle \bx_{j_i}, \bw_2'\rangle|\right)^2 \nonumber\\
& \le L^4 (\|\bw_1-\bw_1'\|_{\mathcal P^\circ} + \|\bw_2-\bw_2'\|_{\mathcal P^\circ})^2 \sum_{i=1}^m \frac{1}{p_{j_i}} \left(|\langle \bx_{j_i}, \bw_1\rangle| + |\langle \bx_{j_i}, \bw_2\rangle| + |\langle \bx_{j_i}, \bw_1'\rangle| + |\langle \bx_{j_i}, \bw_2'\rangle|\right)^2 \nonumber\\
& \le L^4 (\|\bw_1-\bw_1'\|_{\mathcal P^\circ} + \|\bw_2-\bw_2'\|_{\mathcal P^\circ})^2 \cdot 4 \sum_{i=1}^m \frac{1}{p_{j_i}} \left( \langle \bx_{j_i}, \bw_1 \rangle^2 + \langle \bx_{j_i}, \bw_2 \rangle^2 + \langle \bx_{j_i}, \bw_1' \rangle^2 + \langle \bx_{j_i}, \bw_2' \rangle^2 \right)
\nonumber\\
&\le  16L^4(\|\bw_1-\bw_1'\|_{\mathcal P^\circ} + \|\bw_2-\bw_2'\|_{\mathcal P^\circ})^2 \sup_{\bw \in B(R)} \sum_{i=1}^m \frac{1}{p_{j_i}} 
 \langle \bx_{j_i}, \bw \rangle^2, 
 \label{eqn:fixed-metric-bound-tmp}
\end{align}
where the third inequality follows from Cauchy-Schwarz. To simplify the last factor, we note that
\begin{align*}
\sup_{\bw \in B(R)} \sum_{i=1}^m \frac{1}{p_{j_i}} 
 \langle \bx_{j_i}, \bw \rangle^2 
& = \sup_{\bw \in B(R) }\bw^\top \left (\sum_{i=1}^m \tfrac{1}{p_{j_i}}\bx_{j_i} \bx_{j_i}^\top \right)\bw \\
& = R^2 \left\|\sum_{i=1}^m  \frac{1}{p_{j_i}} \bx_{j_i}\bx_{j_i}^\top\right\|.
\end{align*}
Plugging this back into \eqref{eqn:fixed-metric-bound-tmp} and taking square roots yield the claimed result.
\end{proof}

The rest of the proof of Lemma~\ref{lem:fixed_concentration} goes as follows. In Section~\ref{subsec:apd:fixed_expectation}, we show how to derive a bound in expectation of the supremum of $Z(\bw_1,\bw_2)$ using Dudley's inequality and duality of metric entropy. Once this is done, we show how to turn the bound in expectation to a tail bound using a standard argument by concentration of measure in Section~\ref{subsec:apd:fixed_tail}.
\subsection{Bound in expectation via duality}
\label{subsec:apd:fixed_expectation}
We apply Dudley's inequality to the process $Z(\bw_1,\bw_2)$ over the set $B(R) \times B(R)$. Since $Z(\bw, \bw)=0$ for any $\bw$, it follows that $\inf_{\bw_1\bw_2 \in B(R)} |Z(\bw_1,\bw_2)| = 0$. Taking expectation with respect to the Rademacher random variables $\xi_1,\xi_2,\ldots, \xi_m$, we have

\begin{align}
\phantom{\lesssim}\bbE_\xi\sup_{\bw_1,\bw_2\in B(R)} |Z(\bw_1, \bw_2)|
&\lesssim \int_0^\infty\sqrt{\log\Net(B(R)\times B(R), \rho_f, \eps)}~\mathrm{d}\eps
\nonumber\\
&\lesssim L^2 R \left\|\sum_{i=1}^m \frac{1}{p_{j_i}} \bx_{j_i}\bx_{j_i}^\top\right\|^{1/2}\int_0^\infty \sqrt{\log\Net(B(R), \|\cdot\|_{\mathcal P^\circ}, \eps)}~\mathrm{d}\eps,
\label{eqn:fixed_by_dudley}
\end{align}

The second inequality can be verified using Lemma~\ref{lem:fixed-metric-bound} along with standard properties of covering numbers (Lemma~\ref{lem:sublinear-gamma2} and Lemma~\ref{lem:product-gamma2}). 

We move on to estimate the Dudley integral $\int_0^\infty \sqrt{\log\Net(B(R), \|\cdot\|_{\mathcal P^\circ}, \eps)}~\mathrm{d}\eps$.
Since this integral only involves the entropy in dual norm $\|\cdot\|_{\mathcal P^\circ}$, we control it using duality of metric entropy. In particular, one can invoke dual Sudakov minoration (Lemma~\ref{lem:sudakov}) to prove
\begin{equation}
\label{eqn:sudakov}
\int_0^\infty \sqrt{\log\Net(B(R), \|\cdot\|_{\mathcal P^\circ}, \eps)}~\mathrm{d}\eps
\lesssim  \sup_{i\in[m]}\frac{\|\bx_{j_i}\|}{\sqrt{p_{j_i}}} R\sqrt{\log^2\!d\cdot\log m} = R\sqrt{d\log^2\!d\cdot\log m},
\end{equation}
where the last equality uses $\sqrt{p_{j_i}}=\|\bx_{j_i}\|/\sqrt{d}$ which follows from the definition $p_{j_i}=\frac{\tau_{j_i}(\bX)}{\sum_{j=1}^n\tau_j(\bX)}$ and~\ref{eqn:orth_score}. Plugging this into \eqref{eqn:fixed_by_dudley}, we obtain
\begin{align}
\bbE_\xi\sup_{\bw_1,\bw_2\in B(R)} |Z(\bw_1, \bw_2)|
&\lesssim L^2 R^2\left\|\sum_{i=1}^m \frac{1}{p_{j_i}} \bx_{j_i}\bx_{j_i}^\top\right\|^{1/2} \sqrt{d\log^2\!d\cdot\log m}.
\label{eqn:fixed_cond_expectation}
\end{align}

Taking expectation w.r.t. $\bS$ and noticing that, by matrix Chernoff bound \citep{rudelson2007sampling}, 
\begin{equation}
\bbE\left\|\sum_{i=1}^m \frac{1}{p_{j_i}} \bx_{j_i}\bx_{j_i}^\top\right\|\lesssim m + d\log d \lesssim m,
\label{eqn:matrix-chernoff}
\end{equation}
(where the last inequality follows from the assumption $m\gtrsim d\log^3\!d$) we obtain
\[
\E\sup_{\bw_1,\bw_2\in B(R)}|Z(\bw_1, \bw_2)| \lesssim L^2 R^2\sqrt{md\log^2\!d\cdot\log m}.
\]
Under the assumption $m\gtrsim L^4\eps^{-2}d\log^3\!d$, we deduce that 
\begin{equation}
\label{eqn:fixed_expectation_bound}
\frac1m\bbE\sup_{\bw_1,\bw_2\in B(R)} |Z(\bw_1, \bw_2)|
\lesssim \eps R^2.
\end{equation}
This establishes the bound in expectation as desired. Once we can show that the above also holds with high probability rather than in expectation, the proof of Lemma~\ref{lem:fixed_concentration} will be completed. This will be done in the next part.

\subsection{Completing the proof: tail bound via concentration of measure}
\label{subsec:apd:fixed_tail}
This part is more or less standard applications of chaining and concentration of measure arguments \citep{ledoux2001concentration, talagrand2022upper}.
Since $Z(\bw_1,\bw_2)$ conditioned on $\{j_1,\ldots, j_m\}$ is sub-Gaussian with respect to metric $\rho_f$, we have \citep{talagrand2022upper}
\begin{equation}
\label{eqn:fixed_concentration}
\left(\bbE_\xi \sup_{\bw_1,\bw_2\in B(R)} |Z(\bw_1,\bw_2)|^\ell\right)^{1/\ell} \lesssim \bbE_\xi \sup_{\bw_1,\bw_2\in B(R)} |Z(\bw_1,\bw_2)| + \sqrt{\ell}\cdot \Diam(B(R)\times B(R), \rho_f).
\end{equation}
To bound the diameter $\Diam(B(R)\times B(R), \rho_f)$, we invoke Lemma~\ref{lem:fixed-metric-bound} and the simple observation that for $\bw_1,\bw_2,\bw_1',\bw_2'\in B(R)$, 
\[
\|\bw_1-\bw_1'\|_{\mathcal P^\circ} + \|\bw_2-\bw_2'\|_{\mathcal P^\circ}\le 4R\sup_{i\in[m]}\frac{\|\bx_{j_i}\|}{\sqrt{p_{j_i}}} = 4R\sqrt{d},
\]
where the first inequality follows from the definition of $\|\cdot\|_{\mathcal P^\circ}$ and Cauchy-Schwarz, and the second inequality follows from $\sqrt{p_{j_i}}=\|\bx_{j_i}\|/\sqrt{d}$ aforementioned. Plug this into Lemma~\ref{lem:fixed-metric-bound} to obtain
\[
\Diam(B(R)\times B(R), \rho_f) \lesssim L^2 R^2 \sqrt{d}\left\|\sum_{i=1}^m  \frac{1}{p_{j_i}} \bx_{j_i}\bx_{j_i}^\top\right\|^{1/2}.
\]
Plug this and \eqref{eqn:fixed_cond_expectation} into \eqref{eqn:fixed_concentration} to obtain
\[
\left(\bbE_\xi \sup_{\bw_1,\bw_2\in B(R)} |Z(\bw_1,\bw_2)|^\ell\right)^{1/\ell} \lesssim L^2 R^2 \left(\sqrt{d\log^2\!d \cdot \log m} + \sqrt{\ell d}\right) \left\|\sum_{i=1}^m  \frac{1}{p_{j_i}} \bx_{j_i}\bx_{j_i}^\top\right\|^{1/2}.
\]
Applying matrix Chernoff inequality again (but in its tail bound form this time), one has
\[
\left(\E\left\|\sum_{i=1}^m  \frac{1}{p_{j_i}} \bx_{j_i}\bx_{j_i}^\top\right\|^{\ell/2} \right)^{2/\ell} \lesssim m + \sqrt{\ell md\log d}.
\]
Therefore
\begin{align*}
\left( \bbE\sup_{\bw_1,\bw_2\in B(R)} \left| Z(\bw_1, \bw_2) \right|^\ell\right)^{1/\ell}
& = \left( \bbE_{\bS}\bbE_{\xi}\sup_{\bw_1,\bw_2\in B(R)} \left| Z(\bw_1, \bw_2) \right|^\ell\right)^{1/\ell}
\\
& \lesssim \left( \bbE_{\bS} \left( L^2 R^2 \left(\sqrt{d\log^2\!d \cdot \log m} + \sqrt{\ell d}\right) \left\|\sum_{i=1}^m  \frac{1}{p_{j_i}} \bx_{j_i}\bx_{j_i}^\top\right\|^{1/2} \right)^\ell \right)^{1/\ell}
\\ 
& \lesssim L^2 R^2 \left(\sqrt{d\log^2\!d \cdot \log m} + \sqrt{\ell d}\right) \cdot \left(\E\left\|\sum_{i=1}^m  \frac{1}{p_{j_i}} \bx_{j_i}\bx_{j_i}^\top\right\|^{\ell/2} \right)^{1/\ell}
\\
& \lesssim L^2 R^2 \left(\sqrt{d\log^2\!d \cdot \log m} + \sqrt{\ell d}\right) \left(m + \sqrt{\ell md\log d}\right)^{1/2}.
\end{align*}
For the last expression, using $(a+b)^{1/2}\lesssim a^{1/2} + b^{1/2}$ for $a,b>0$ and expanding the product, we obtain
\begin{align*}
& \left( \bbE\sup_{\bw_1,\bw_2\in B(R)} \left| Z(\bw_1, \bw_2) \right|^\ell\right)^{1/\ell} \\
& \quad \lesssim L^2 R^2 \Big( \sqrt{md\log^2\!d\cdot\log m} 
+ \ell^{\frac14} m^{\frac14}d^{\frac34}\log^{\frac54}\!d\cdot\log^{\frac12}\!m + \ell^{\frac12}\sqrt{md} 
+ \ell^{\frac34}m^{\frac14}d^{\frac34}\log^{\frac14}\!d
\Big).
\end{align*}

This gives the desired tail bound by a standard computation based on Markov's inequality, which is summarized in the following lemma.
\begin{lemma}
\label{lem:moment_to_tail}
Assume a random variable $X$ satisfies
\[
\left(\bbE|X|^\ell\right)^{1/\ell} \le A_0 + \sum_{k=1}^K A_k \ell^{\alpha_k}, \quad \forall\ell\ge 1.
\]
where $A_k\ge 0$, $\alpha_k\in(0, 1]$ are constants. There is some constant $C>0$ depending only on $K$, such that
for any $\delta\in(0,1/2)$ we have
\[
\bbP\left(|X| > CA_0 + C\sum_{k=1}^K A_k\log^{\alpha_k}\Big(\frac1{\delta}\Big) \right) \le 1-\delta.
\]
\end{lemma}

\subsection{Proof of Lemma \ref{lem:chaining}}
\label{subsec:apd:symmetrization}
Throughout this proof, we denote 
\[\bv(\bw_1,\bw_2) \coloneqq f(\bX\bw_1) - f(\bX \bw_2).\]
Let $\bS'$ be an independent copy of $\bS$. Recalling the definition of $\bS$, we denote by $j_i$ the indices chosen by $\bS$ and $j'_i$ the indices chosen by $\bS'$. We want to estimate the following
\[E \coloneqq \bbE \sup_{\bw_1,\bw_2 \in B(R)} \left \lvert \norm{\bS \bv(\bw_1,\bw_2)}^2 - \norm{\bv(\bw_1,\bw_2)}^2\right \rvert^\ell.
\]

Since $\E \norm{\bS' \bv(\bw_1,\bw_2) }^2 - \norm{\bv(\bw_1,\bw_2)}^2 = 0$, we have
\begin{align*}
  E &= \E_{\bS} \sup_{\bw_1,\bw_2\in B(R)} \left \lvert\norm{\bS \bv(\bw_1,\bw_2)}^2 - \norm{\bv(\bw_1,\bw_2)}^2 - \E_{\bS'} \left(\norm{\bS' \bv(\bw_1,\bw_2)}^2 - \norm{\bv(\bw_1,\bw_2)}^2\right)\right \rvert^\ell\\
  &= \E_{\bS} \sup_{\bw_1,\bw_2\in B(R)} \left \lvert \norm{\bS \bv(\bw_1,\bw_2)}^2 - \E_{\bS'}\norm{\bS' \bv(\bw_1,\bw_2)}^2\right \rvert^\ell.
\end{align*}
From the convexity of $\abs{\cdot}^\ell$, we get
\begin{align*}
  E &\le \E_{\bS} \sup_{\bw_1,\bw_2\in B(R)} \E_{\bS'} \abs{\norm{\bS \bv(\bw_1,\bw_2)}^2 - \norm{\bS'\bv(\bw_1,\bw_2)}^2}^\ell\\
  &\le \E_{\bS, \bS'} \sup_{\bw_1,\bw_2\in B(R)} \abs{\norm{\bS \bv(\bw_1,\bw_2)}^2 - \norm{\bS'\bv(\bw_1,\bw_2)}^2}^\ell\\
  &= \E \sup_{\bw_1,\bw_2 \in B(R)} \abs{\frac1m \sum_{i=1}^m \frac{(v_{j_i}(\bw_1,\bw_2))^2}{p_{j_i}} - \frac1m \sum_{i=1}^m \frac{(v_{j'_i}(\bw_1,\bw_2))^2}{p_{j'_i}}}^\ell.
\end{align*}
Since $\sum_{i=1}^m \frac{(v_{j'_i}(\bw_1,\bw_2))^2}{p_{j'_i}}$ is an independent copy of $\sum_{i=1}^m \frac{(v_{j_{i}}(\bw_1,\bw_2))^2}{p_{j_{i}}}$, their difference is a symmetric distribution and the same as $\sum_{i=1}^m \xi_i \left(\frac{(v_{j_i}(\bw_1,\bw_2))^2}{p_{j_i}} - \frac{(v_{j'_i}(\bw_1,\bw_2))^2}{p_{j'_i}}\right)$, where $\xi_1,\xi_2,\ldots, \xi_m$ are i.i.d. Rademacher random variables. Therefore
\begin{align*}
E &\le \E \sup_{\bw_1,\bw_2 \in B(R)} \abs{\frac1m \sum_{i=1}^m \xi_i \left(\frac{(v_{j_i}(\bw_1,\bw_2))^2}{p_{j_i}} - \frac{(v_{j'_i}(\bw_1,\bw_2))^2}{p_{j'_i}}\right)}^\ell\\
& \le 2^{\ell-1}\cdot\E \sup_{\bw_1,\bw_2 \in B(R)} \abs{\sum_{i=1}^m \xi_i \frac{(v_{j_i}(\bw_1,\bw_2))^2}{p_{j_i}}}^\ell
+ 2^{\ell-1}\cdot\E \sup_{\bw_1,\bw_2 \in B(R)} \abs{\frac1m \sum_{i=1}^m \xi_i \frac{(v_{j'_i}(\bw_1,\bw_2))^2}{p_{j'_i}}}^\ell\\
&=2^\ell \cdot\E \sup_{\bw_1,\bw_2 \in B(R)} \abs{\sum_{i=1}^m \xi_i \frac{(v_{j_i}(\bw_1,\bw_2))^2}{p_{j_i}}}^\ell,
\end{align*}
where the second line is the triangle inequality $|a+b|^\ell\le 2^{\ell-1}(|a|^\ell + |b|^\ell)$, and the last line used the observation that the two terms are identically distributed.

\subsection{Proof of Lemma \ref{lem:moment_to_tail}}
By Markov's inequality, we have, for all $t>0$ and $\ell\ge 1$, that
\begin{align*}
\bbP(|X| > t) 
& \le t^{-\ell} \left(A_0 + \sum_{k=1}^K A_k \ell^{\alpha_k} \right)^{\ell}
\\
& \le C_K^\ell t^{-\ell} \left(A_0^\ell + \sum_{k=1}^K A_k^\ell \ell^{\alpha_k \ell}\right)\\
& = \left(\frac{C_K A_0}{t}\right)^\ell + \sum_{k=1}^K \left(\frac{C_K A_k \ell^{\alpha_k}}{t}\right)^\ell.
\end{align*}
In particular, setting $t = 4(K+1)C_K A_0 + 4(K+1)\sum_{k=1}^K A_k \ell^{\alpha_k}$, we obtain
\[
\bbP\left(|X| > 4(K+1)C_K A_0 + 4(K+1)C_K\sum_{k=1}^K A_k \ell^{\alpha_k}\right)
\le (4(K+1))^{-\ell} + \sum_{k=1}^K (4(K+1))^{-\ell} \le 4^{-\ell}.
\]
The desired conclusion follows from plugging in $\ell = \lceil\log(1/\delta)\rceil$.

\section{Proof of main theorem with unknown nonlinearity, Theorem \ref{thm:unknown_f}}
\label{subsec:apd:unknown}

\begin{proof}
Similar to the proof of Lemma~\ref{lem:fixed_concentration}, we need to bound the supremum of the process
\[
Z_{f_1,f_2}(\bw_1, \bw_2)\coloneqq \sum_{i=1}^m \xi_i \frac{\big( f_1(\langle \bx_{j_i}, \bw_1\rangle) - f_2(\langle\bx_{j_i}, \bw_2\rangle \big)^2}{p_{j_i}},
\qquad f_1,f_2\in\Lip_L,~(\bw_1,\bw_2)\in B(R)\times B(R).
\]

The crucial step is our proof is to construct an appropriate discretization of $\Lip_L$. The guideline for choosing such a net will be based on the seminal idea of chaining for Bernoulli process \citep{talagrand2022upper}. In our setting, it suffices to use a simplified version of chaining, which we shall present immediately after introducing the relevant notations. Define for $f_1, f_2\in\Lip_L$ the following two metrics:
\begin{align*}
&D_\infty\bigg((f_1, f_2), (f_1', f_2')\bigg) \\
&\quad \coloneqq \sup_{\bw_1,\bw_2\in B(R)} \sum_{i} \frac 1 {p_{j_i}} \left| \left( f_1(\langle \bx_{j_i}, \bw_1\rangle) - f_2(\langle \bx_{j_i}, \bw_2\rangle) \right)^2 - \left( f_1'(\langle \bx_{j_i}, \bw_1\rangle) - f_2'(\langle \bx_{j_i}, \bw_2\rangle) \right)^2 \right|
\end{align*}
and 
\begin{align*}
&D_2\bigg((f_1, f_2), (f_1', f_2')\bigg)\\
&\quad \coloneqq \left( \sup_{\bw_1,\bw_2\in B(R)} \sum_{i} \frac 1 {p_{j_i}^2} \left| 
\left( f_1(\langle \bx_{j_i}, \bw_1\rangle) - f_2(\langle \bx_{j_i}, \bw_2\rangle) \right)^2 - \left( f_1'(\langle \bx_{j_i}, \bw_1\rangle) - f_2'(\langle \bx_{j_i}, \bw_2\rangle) \right)^2 \right|^2 \right)^{1/2}.
\end{align*}

\begin{lemma}[Chaining for Bernoulli process]\label{lem:chaining}
Let $\Net_\Delta\subset \Lip_L$ be such that $\Net_\Delta\times\Net_\Delta$ is a $\Delta$-net  of $\Lip_L\times\Lip_L$ w.r.t metric $D_\infty$ (where $\Delta$ can be a random variable depending on $\bS$; more precisely, it is measurable with respect to the $\sigma$-algebra generated by $\bS$). Then
\begin{align*}
& \E \sup_{\substack{f_1,f_2\in\Lip_L \\ \bw_1, \bw_2\in B(R)}} \left| \sum_{i=1}^m \frac{\xi_i}{p_{j_i}} \left( f_1(\langle \bx_{j_i}, \bw_1\rangle) - f_2(\langle \bx_{j_i}, \bw_2\rangle) \right)^2 \right|
\\
& \quad \lesssim L^2 R^2\sqrt{md\log^2\!d \cdot \log m} + \E \Delta + \E  \int_0^\infty\sqrt{\log \Net(\Net_\Delta\times \Net_\Delta, D_2, \eps)}~\mathrm{d}\eps.
\end{align*}
\end{lemma}
The proof is postponed to Section~\ref{subsec:apd:pf_chaining}. 

The metrics $D_\infty$ and $D_2$, which are defined on $\Lip_L\times \Lip_L$, are highly complicated to analyse. So, we consider the following upper bounds of them, which decompose them into simpler metrics defined on $\Lip_L$.
\begin{lemma}[Decomposing the metric $D_\infty$ and $D_2$]
\label{lem:D_bounds}
Let $I_{j_i}=[-R\|\bx_{j_i}\|, R\|\bx_{j_i}\|]$. Then
\[
D_\infty\bigg((f_1, f_2), (f_1', f_2')\bigg) \lesssim \rho_{\infty}(f_1, f_1') + \rho_{\infty}(f_2, f_2'),
\]
where
\[ \rho_\infty(f_1, f_2)
\coloneqq LR \left\|\sum_{i=1}^m \frac{1}{p_{j_i}} \bx_{j_i}\bx_{j_i}^\top\right\|^{1/2} 
\left( \sum_{i=1}^m \frac{1}{p_{j_i}} \|f_1 - f_2\|^2_{L^\infty(I_{j_i})} \right)^{1/2}.
\]  
On the other hand, for any $\mu>0$, we have
\[
D_2\bigg((f_1, f_2), (f_1', f_2')\bigg) \lesssim \rho_2(f_1, f_1') + \rho_2(f_2, f_2'),
\]
where the universal constant hidden by $\lesssim$ is independent of $\mu$, and the metric $\rho_2$ is defined by
\begin{align*}
&\rho_2(f_1, f_2)\\
&\quad\coloneqq 
L R\left\|\sum_{i=1}^m \frac{1}{p_{j_i}} \bx_{j_i}\bx_{j_i}^\top\right\|^{1/2}\left( \sup_{i\in[m]} \frac{\|f_1-f_2\|_{L^\infty(I_{j_i})} \mathbf{1}_{\|\bx_{j_i}\|\ge\mu} }{\sqrt{p_{j_i}}} \right)
+ L^2 R^2 d \left(\sum_{i=1}^m \mathbf{1}_{\|\bx_{j_i}\|<\mu}\right)^{1/2} \rho_\delta(f_1, f_2).
\end{align*}
Here $\rho_\delta$ is the Dirac distance,
\[
\rho_\delta(f_1, f_2) = \begin{cases}
    0, & f_1= f_2,\\
    1, & f_1\ne f_2.
\end{cases}
\]
\end{lemma}
The proof can be found in Appendix~\ref{apd:lem:D_bounds}.

As an immediate application of the bounds in Lemma~\ref{lem:D_bounds}, we obtain the following simplified version of the Lemma~\ref{lem:chaining}.
\begin{corollary}
[Chaining with simplified metrics].
Assume $\Net_\Delta\subset \Lip_L$ is a $\Delta$-net of $\Lip_L$ w.r.t metric $\rho_\infty$ (where $\Delta$ can be a random variable depending on $\bS$). Then
\begin{align*}
& \E \sup_{\substack{f_1,f_2\in\Lip_L \\ \bw_1, \bw_2\in B(R)}} \left| \sum_{i=1}^m \frac{\xi_i}{p_{j_i}} \left( f_1(\langle \bx_{j_i}, \bw_1\rangle) - f_2(\langle \bx_{j_i}, \bw_2\rangle) \right)^2 \right|
\\
& \quad \lesssim L^2 R^2\sqrt{md\log^2\!d \cdot \log m} + \E \Delta + \E  \int_0^\infty\sqrt{\log \Net(\Net_\Delta, \rho_2, \eps)}~\mathrm{d}\eps.
\end{align*}    
\end{corollary}

We turn to construct a $\Delta$-net for $\Lip_L$ with respect to the metric $\rho_\infty$. 
In light of the definition of $\rho_\infty$, one may try to construct an $\Delta$-net with respect to $\rho_\infty$ by piecewise linear functions, each of which differs with its nearest neighbors on the interval $I_{j_i}$ by an amount proportional to $\sqrt{p_{j_i}}$, say $\eta\sqrt{p_{j_i}}$ for some $\eta>0$ to be chosen later. As long as $p_{j_i}$ is not too small, this is achievable. This idea culminates to the following lemma.

\begin{lemma}[Construction of $\Net_\Delta$]
\label{lem:D-infty-net}
Fix some $\mu, \eta\in (0, 1/2)$. 
Let 
\[
\Delta(\mu, \eta) \coloneqq \eta L^2 R^2 \left\|\sum_{i=1}^m \frac{1}{p_{j_i}} \bx_{j_i}\bx_{j_i}^\top\right\|^{1/2}
\left( md + \sum_{i=1}^m \frac{\mu^2}{p_{j_i}} \right)^{1/2}.
\]
Then there exists a $\Delta(\mu,\eta)$-net of $\Lip_L$ with respect to the metric $\rho_\infty$, such that
\begin{enumerate}[label=(\roman*)]
  \item The cardinality of $\Net_\Delta$ is controlled:
  \[
    \log|\Net_\Delta|\lesssim \frac{\log(1/\mu)}{\eta}.
  \]
  \item Assume $m\ge Cd\log d$. The expectation of the Dudley's integral of $\Net_\Delta$ with respect to $\rho_2$ is also controlled:
  \[
  \E \int_0^\infty\sqrt{\log \Net(\Net_\Delta, \rho_2, \eps)}~\mathrm{d}\eps \lesssim L^2 R^2 \sqrt{md} \left(\log(1/\mu) + \mu \sqrt{\frac{n\log(1/\mu)}{\eta}}\right).
  \]
  \item Assume $m\ge Cd\log d$. The expectation of $\Delta$ satisfies
  \[
  \E \Delta(\mu, \eta) \lesssim \eta L^2 R^2 m\sqrt{d}\cdot\sqrt{1 + \frac{\mu^2 n}{d}}.
  \]
\end{enumerate}
\end{lemma}
The proof of this lemma, which is a major technical challenge in this paper, is postponed to Appendix~\ref{subsec:apd:pf-D-infty-net}.

Plug this into Lemma~\ref{lem:chaining} to obtain
\begin{align*}
&\bbE \sup_{\substack{f_1, f_2\in\Lip_L\\\bw_1,\bw_2\in B(R)}} |Z_{f_1, f_2}(\bw_1,\bw_2)|
\\
&\lesssim 
L^2 R^2\sqrt{md\log^2\!d \cdot \log m} + \eta L^2 R^2 m\sqrt{d}\cdot\sqrt{1 + \frac{\mu^2 n}{d}}  + L^2 R^2 \sqrt{md} \left(\log(1/\mu) + \mu \sqrt{\frac{n\log(1/\mu)}{\eta}}\right).
\end{align*}
Setting $\mu=\frac{d^2}{nm}$ and $\eta=1/\sqrt{m}$, we obtain
\begin{align*}
&\bbE \sup_{\substack{f_1, f_2\in\Lip_L\\\bw_1,\bw_2\in B(R)}} |Z_{f_1,f_2}(\bw_1,\bw_2)|
\\
&\quad \lesssim 
L^2 R^2\sqrt{md\log^2\!d \cdot \log m} + L^2 R^2 \sqrt{md}  + L^2 R^2 \sqrt{md} \log\left(\frac{nm}{d^2}\right).
\end{align*}
From this, it can be seen that whenever $m\gtrsim L^4\eps^{-2}d(\log^2 n + \log^3\!d)$, we have 
\[\bbE  \sup_{\substack{f_1,f_2\in\Lip_L\\\bw_1,\bw_2\in B(R)}} |Z_{f_1,f_2}(\bw_1,\bw_2)|\le \eps R^2. \]
Similar to the proof of Lemma~\ref{lem:fixed_concentration}, this is close to what we desire, except that we need a tail bound. The latter can be deduced from the above using the same concentration of measure argument as in Lemma~\ref{lem:fixed_concentration}, which we omit here to avoid repetition. 
\end{proof}
\subsection{Proof of Lemma \ref{lem:chaining}}
\label{subsec:apd:pf_chaining}

We begin with an important property of $D_\infty$ that the following error bound is true deterministically:
\[
\left|Z_{f_1, f_2}(\bw_1, \bw_2) - Z_{f_1', f_2'}(\bw_1, \bw_2)\right|
\le D_\infty\big((f_1, f_2), (f_1', f_2')\big).
\]
This follows easily from the fact that $|\xi_{i}|\le 1$. 
Therefore, if $\Net_\Delta\times\Net_\Delta$ is a $\Delta$-net of $\Lip_L\times\Lip_L$ with respect to the metric $D_\infty$, one has
\begin{equation}
\sup_{f_1, f_2\in\Lip_L} |Z_{f_1, f_2}(\bw_1,\bw_2)| \le \sup_{f_1, f_2\in\Net_\Delta} |Z_{f_1, f_2}(\bw_1,\bw_2)| + \Delta.
\label{eqn:sup-with-Dinfty}    
\end{equation}
Taking supremum with respect to $\bw_1, \bw_2$ and then taking expectation, we obtain
\begin{equation}
\label{eqn:chaining-Dinfty}
\E \sup_{\substack{f_1, f_2\in\Lip_L \\ \bw_1, \bw_2\in B(R)}} |Z_{f_1, f_2}(\bw_1, \bw_2)|
\le \E\sup_{\substack{f_1, f_2\in\Net_\Delta \\ \bw_1, \bw_2\in B(R)}}|Z_{f_1, f_2}(\bw_1, \bw_2)| + \E \Delta. 
\end{equation}
To bound the expectation of the supremum appearing in the right hand side, we condition on $j_1,\ldots,j_m$ again and observe that $Z_{f_1, f_2}(\bw_1,\bw_2)$ is sub-Gaussian over index $(f_1, f_2, \bw_1, \bw_2)\in\Net_\Delta\times\Net_\Delta\times B(R)\times B(R)$, endowed with metric
\begin{align*}
& D\big( (f_1, f_2, \bw_1, \bw_2), (f_1', f_2', \bw_1', \bw_2') \big) \\
&\quad \coloneqq 
\left(\sum_{i} \frac 1 {p_{j_i}^2} \left| 
\left( f_1(\langle \bx_{j_i}, \bw_1\rangle) - f_2(\langle \bx_{j_i}, \bw_2\rangle) \right)^2 - \left( f_1'(\langle \bx_{j_i}, \bw_1'\rangle) - f_2'(\langle \bx_{j_i}, \bw_2'\rangle) \right)^2 \right|^2 \right)^{1/2}.
\end{align*}
Applying triangle inequality gives
\begin{align*}
&D\big( (f_1, f_2, \bw_1, \bw_2), (f_1', f_2', \bw_1', \bw_2') \big)\\
&\quad \le D\big( (f_1, f_2, \bw_1', \bw_2'), (f_1', f_2', \bw_1', \bw_2') \big) + D\big( (f_1, f_2, \bw_1, \bw_2), (f_1, f_2, \bw_1', \bw_2') \big).
\end{align*}
For the first term, by definition, it's easy to see that 
\begin{align*}
&\phantom{=} D\big( (f_1, f_2, \bw_1', \bw_2'), (f_1', f_2', \bw_1', \bw_2') \big)\\
&\le \sup_{\bw_1,\bw_2\in B(R)} D\big( (f_1, f_2, \bw_1, \bw_2), (f_1', f_2', \bw_1, \bw_2) \big)\\
&= D_2\big((f_1, f_2), (f_1', f_2')\big),
\end{align*}
and, for the second term, it is easy to verify that
\[
 D\big( (f_1, f_2, \bw_1, \bw_2), (f_1, f_2, \bw_1', \bw_2') \big) = \rho_{f_1-f_2}\big( (\bw_1, \bw_2),  (\bw_1', \bw_2') \big) \le 2\rho\big( (\bw_1, \bw_2),  (\bw_1', \bw_2') \big),
\]
where $\rho_f$ was defined in \eqref{eqn:rho}, the last inequality follows from Lemma~\ref{lem:fixed-metric-bound} (noticing that $f_1-f_2\in\Lip_{2L}^0$).

These inequalities together imply
\[
D\big( (f_1, f_2, \bw_1, \bw_2), (f_1', f_2', \bw_1', \bw_2') \big)
\le D_2\big((f_1, f_2), (f_1', f_2')\big) + 2\rho\big( (\bw_1, \bw_2),  (\bw_1', \bw_2') \big).
\]

Bearing the above inequality in mind,
we now apply Dudley's inequality (Lemma~\ref{lem:dudley}) to the sub-Gaussian process $Z_{f_1, f_2}(\bw_1,\bw_2)$ over $f_1, f_2\in\Net_\Delta$, $\bw_1,\bw_2\in B(R)$ endowed with metric $D$ (which is sub-Gaussian with respect to the randomness of $\xi$) and invoke Lemma~\ref{lem:product-gamma2} to obtain
\begin{align*}
\E_{\xi}\sup_{\substack{f_1, f_2\in\Lip_L \\ \bw_1, \bw_2\in B(R)}} |Z_f (\bw_1, \bw_2)| 
& \lesssim  \int_0^\infty\sqrt{\log \Net(\Net_\Delta\times\Net_\Delta, D_2, \eps)}~\mathrm{d}\eps + \int_0^\infty\sqrt{\log\Net(B(R)\times B(R), \rho,\eps)}~\mathrm{d}\eps
\\
& \lesssim \int_0^\infty\sqrt{\log \Net(\Net_\Delta\times\Net_\Delta, D_2, \eps)}~\mathrm{d}\eps + L^2 R^2 \sqrt{d\log^2\!d \cdot\log m}\left\|\sum_{i=1}^m \frac{1}{p_{j_i}} \bx_{j_i}\bx_{j_i}^\top\right\|^{1/2},
\end{align*}
where the last line uses \eqref{eqn:fixed_by_dudley} and \eqref{eqn:sudakov}. 
Put this and \eqref{eqn:chaining-Dinfty} together, and then take expectation with respect to the randomness of $j_1,\cdots,j_m$, the desired conclusion readily follows.

\subsection{Proof of Lemma \ref{lem:D_bounds}}
\label{apd:lem:D_bounds}
\subsubsection{Controlling $D_\infty$}
First note that the term inside the summation of $D_\infty\big((f_1,f_2), (f_1', f_2')\big)$ can be written as
\begin{align*}
& \phantom{=} \left| \left( f_1(\langle \bx_{j_i}, \bw_1\rangle) - f_2(\langle \bx_{j_i}, \bw_2\rangle \right)^2 - \left( f_1'(\langle \bx_{j_i}, \bw_1\rangle) - f_2'(\langle \bx_{j_i}, \bw_2\rangle \right)^2 \right|
\\
& = \underbrace{\left| f_1(\langle \bx_{j_i}, \bw_1\rangle) - f_1'(\langle \bx_{j_i}, \bw_1\rangle) - f_2(\langle \bx_{j_i}, \bw_2\rangle) + f_2'(\langle \bx_{j_i}, \bw_2\rangle) \right|}_{\eqqcolon T_1} 
\\
&\quad \cdot \underbrace{\left| f_1(\langle \bx_{j_i}, \bw_1\rangle) - f_2(\langle \bx_{j_i}, \bw_2\rangle) + f_1'(\langle \bx_{j_i}, \bw_1\rangle) - f_2'(\langle \bx_{j_i}, \bw_2\rangle)  \right|}_{\eqqcolon T_2}.
\end{align*}
The first factor $T_1$ can be controlled in the following way. Recall that $\|\bw_k\|\le R$ for $k=1,2$, we have $|\langle \bx_{j_i}, \bw_k\rangle|\le R\|\bx_{j_i}\|$, thus $\langle \bx_{j_i}, \bw_k\rangle\in I_{j_i}$. Therefore
\[
\left| f_k(\langle \bx_{j_i}, \bw_k\rangle) - f_k'(\langle \bx_{j_i}, \bw_k\rangle)
\right|
\le \|f_k - f_k'\|_{L^\infty(I_{j_i})},
\quad k=1,2.
\]
therefore
\[T_1\le \|f_1-f_1'\|_{L^\infty(I_{j_i})} + \|f_2-f_2'\|_{L^\infty(I_{j_i})}.\] 

The second factor $T_2$ can be controlled using the Lipschitz assumption, which implies for example $|f_1(\langle \bx_{j_i}, \bw_1\rangle)|\le |f_1(0)| + L|\langle\bx_{j_i}, \bw_1\rangle| = L|\langle\bx_{j_i}, \bw_1\rangle|$ as $f_1\in\Lip_L$ and similar inequalities for the other terms. Thus
\[T_2\le 2L|\langle\bx_{j_i}, \bw_1\rangle| + 2L|\langle\bx_{j_i}, \bw_2\rangle|.\]

Combining these bounds, we obtain
\begin{align}
&\phantom{\le}\left| \left( f_1(\langle \bx_{j_i}, \bw_1\rangle) - f_2(\langle \bx_{j_i}, \bw_2\rangle) \right)^2 - \left( f_1'(\langle \bx_{j_i}, \bw_1\rangle) - f_2'(\langle \bx_{j_i}, \bw_2\rangle) \right)^2 \right|
\nonumber \\
&\le 
\left( \|f_1-f_1'\|_{L^\infty(I_{j_i})} + \|f_2-f_2'\|_{L^\infty(I_{j_i})} \right)\cdot 2L \left(|\langle \bx_{j_i}, \bw_1\rangle| + |\langle \bx_{j_i}, \bw_2\rangle|\right),
\label{eqn:square-diff-bound}
\end{align}
hence
\begin{align*}
&D_\infty\big((f_1,f_2), (f_1', f_2')\big)\\
& \quad \le 4L \sup_{\bw_1, \bw_2\in B(R)} \sum_{i=1}^m \frac{1}{p_{j_i}} \left(\left|\langle\bx_{j_i}, \bw_1\rangle\right| + \left|\langle\bx_{j_i}, \bw_2\rangle\right|\right) \cdot \left( \|f_1-f_1'\|_{L^\infty(I_{j_i})} + \|f_2-f_2'\|_{L^\infty(I_{j_i})} \right)
\\
& \quad \le 8L\sup_{ \bw\in B(R)} \sum_{i=1}^m \frac{1}{p_{j_i}} \left|\langle\bx_{j_i}, \bw\rangle\right| \left( \|f_1-f_1'\|_{L^\infty(I_{j_i})} + \|f_2-f_2'\|_{L^\infty(I_{j_i})} \right)
\\
& \quad = 8LR \sup_{\|\bw\|\le 1} \sum_{i=1}^m  \left|\left\langle \frac{\bx_{j_i}}{\sqrt{p_{j_i}}}, \bw \right\rangle\right| \cdot \frac{1}{\sqrt{p_{j_i}}} \|f_1-f_1'\|_{L^\infty(I_{j_i})} \\
& \quad \phantom{=} + 8LR \sup_{\|\bw\|\le 1} \sum_{i=1}^m  \left|\left\langle \frac{\bx_{j_i}}{\sqrt{p_{j_i}}}, \bw \right\rangle\right| \cdot \frac{1}{\sqrt{p_{j_i}}}\|f_2-f_2'\|_{L^\infty(I_{j_i})}.
\end{align*}
The conclusion of the lemma then follows from  Cauchy-Schwarz with similar procedures as in the proof of Lemma~\ref{lem:fixed-metric-bound}.

\subsubsection{Controlling $D_2$} 
We break the summation in the definition of $D_2$ into two parts:
\begin{align*}
&D_2^2\big((f_1,f_2), (f_1', f_2')\big)\\
& \quad \lesssim 
\sup_{\bw_1,\bw_2\in B(R)} \underbrace{\sum_{i=1}^m \mathbf{1}_{\|\bx_{j_i}\|\ge \mu} \cdot \frac 1 {p_{j_i}^2} \left| 
\left( f_1(\langle \bx_{j_i}, \bw_1\rangle) - f_2(\langle \bx_{j_i}, \bw_2\rangle) \right)^2 - \left( f_1'(\langle \bx_{j_i}, \bw_1\rangle) - f_2'(\langle \bx_{j_i}, \bw_2\rangle) \right)^2 \right|^2}_{\eqqcolon T_1} \\
& \quad + \sup_{\bw_1,\bw_2\in B(R)} \underbrace{\sum_{i=1}^m \mathbf{1}_{\|\bx_{j_i}\| < \mu} \cdot \frac 1 {p_{j_i}^2} \left| 
\left( f_1(\langle \bx_{j_i}, \bw_1\rangle) - f_2(\langle \bx_{j_i}, \bw_2\rangle) \right)^2 - \left( f_1'(\langle \bx_{j_i}, \bw_1\rangle) - f_2'(\langle \bx_{j_i}, \bw_2\rangle) \right)^2 \right|^2}_{\eqqcolon T_2}. 
\end{align*}
Using inequality~\eqref{eqn:square-diff-bound}, we obtain
\begin{align*}
T_1 
& \lesssim L^2 \sup_{\bw_1,\bw_2\in B(R)} \sum_{i=1}^m \mathbf{1}_{\|\bx_{j_i}\| \ge \mu} \frac{|\langle \bx_{j_i}, \bw_1\rangle|^2 + |\langle \bx_{j_i}, \bw_2\rangle|^2}{p_{j_i}} \cdot \frac{\|f_1-f_1'\|_{L^\infty(I_{j_i})}^2 + \|f_2-f_2'\|_{L^\infty(I_{j_i})}^2}{p_{j_i}}
\\
& \lesssim L^2 R^2 \left\|\sum_{i=1}^m \frac{1}{p_{j_i}} \bx_{j_i}\bx_{j_i}^\top\right\| \left( \sup_{i\in[m]} \frac{ \left( \|f_1 - f_1'\|_{L^\infty(I_{j_i})}^2 + \|f_2 - f_2'\|_{L^\infty(I_{j_i})}^2 \right) \mathbf{1}_{\|\bx_{j_i}\|\ge\mu} }{p_{j_i}} \right)
.
\end{align*}

On the other hand, by Lipschitz property of $f_1$, we have 
\[
|f_1(\langle \bx_{j_i}, \bw_1\rangle)| \le L\|\bx_{j_i}\|\cdot\|\bw_1\|\le LR\|\bx_{j_i}\|,
\]
since $\bw_1\in B(R)$. The same upper bound holds for $|f_2(\langle \bx_{j_i}, \bw_2\rangle)|$, $|f_1'(\langle \bx_{j_i}, \bw_1\rangle)|$, and $|f_2'(\langle \bx_{j_i}, \bw_2\rangle)|$. It is then easy to show
\begin{align*}
T_2 
& \lesssim \sum_{i=1}^m \mathbf{1}_{\|\bx_{j_i}\| < \mu} \cdot \frac{1}{p_{j_i}^2} L^4 R^4 \|\bx_{j_i}\|^4 \left(\rho_\delta(f_1, f_1') + \rho_\delta(f_2, f_2')\right)\\
& = L^4 R^4 d^2 \left( \sum_{i=1}^m \mathbf{1}_{\|\bx_{j_i}\| < \mu} \right) \left(\rho_\delta(f_1, f_1') + \rho_\delta(f_2, f_2')\right),
\end{align*}
where the last equality uses $p_{j_i}=\|\bx_{j_i}\|^2/d$ again. 
The conclusion follows immediately from summing up the above inequalities for $T_1$ and $T_2$ and then taking square roots.

\subsection{Proof of Lemma \ref{lem:D-infty-net}}
\label{subsec:apd:pf-D-infty-net}
We divide the proof into four steps. First, we give an explicit construction for a candidate of the net $\Net_\Delta$ in Section~\ref{subsubsec:apd:net_construction}. Up to this point, we do not know whether $\Net_\Delta$ is a net, but we can already bound the cardinality of it, thereby proving item (i) of Lemma~\ref{lem:D-infty-net}. Then in Section~\ref{subsubsec:apd:net_proof}, we prove that $\Net_\Delta$ is indeed a $\Delta(\mu,\eta)$ net in $\rho_\infty$ distance. On the other hand, 
proving the bound of Dudley integral in item (ii) turns out the trickiest part of the proof. In Section~\ref{subsubsec:apd:dudley_bound}, we will prove this item by presenting a technique to embed $\Net_\Delta$ to another space of Lipschitz functions (endowed with a different metric than $\rho_\infty$), which enables to draw connections to well-known bounds on the entropy of class of Lipschitz functions. Finally, we prove item (iii) in Section~\ref{subsubsec:apd:Delta_expectation}, which follows from Cauchy-Schwarz inequality. 
\subsubsection{Construction of $\Net_\Delta$.}
\label{subsubsec:apd:net_construction}
For convenience, set $\eta'=\eta/4$ and 
\begin{gather*}
K = \left\lceil\frac{\log(1/\mu)}{\log(1+\eta')}\right\rceil,
\qquad N = \left\lceil \frac{1}{\eta'} \right\rceil,
\\
z_k = \begin{cases}
R(1+\eta')^{-(K+N-k)}, & k=N,N+1,\cdots,N+K,\\
\frac{k}{N}z_N, & k=0,1,\cdots,N-1.
\end{cases}
\end{gather*}
The construction of $z_k$ guarantees that $z_k$ is monotone increasing and
\begin{equation}
\label{eqn:z_prop}
z_{k+1} = (1+\eta') z_k, \quad N\le k\le N+K-1.
\end{equation}
and 
\begin{equation}
\label{eqn:z_crit}
\frac12 \mu R< \frac{1}{1+\eta'} \mu R\le z_N \le \mu R.
\end{equation}
The basic idea is to construct at each point  $\pm z_k$ a grid with spacing $L\eta' z_k$ for $[-Lz_k, Lz_k]$, and connect these grid points by piecewise linear functions. However, for small $z_k$ where $k\le N$ this would be impossible without significantly enlarging the size of the net. To remedy this, for such $k$ the spacing of the grid will be $L\eta z_N$. The precise construction is as follows. Set
\[
\Sigma_0 = \Big\{\sigma=(\sigma_{-N-K}, \sigma_{-N-K+1},\ldots,\sigma_{N+K}): \sigma_0=0,~ \sigma_{\pm k}\in\mathbb Z,~k\in[N+K]\Big\}.
\]
To each $\sigma=(\sigma_{-N-K}, \sigma_{-N-K+1},\ldots,\sigma_{N+K})\in \Sigma_0$, we associate a piecewise linear function defined by
\begin{equation}
\label{eqn:net-pl}
f_\sigma(\pm z_{k}) = L\eta' \sigma_{\pm k} \max(z_k, z_N),\quad 0\le k\le N+K.
\end{equation}
The value of $f_\sigma$ at a point which is not equal to any one of $\pm z_k$ is given by linear interpolation. The whole collection of $\{f_\sigma:\sigma\in\Sigma_0\}$ would be too large to work with. Fortunately, since our net is for $L$-Lipschitz functions rather than arbitrary functions, we can restrict our attention to a much smaller subset $\Sigma\subset \Sigma_0$. 
Denote
\begin{equation}
\label{eqn:net-ind}
\Sigma = \left\{\sigma=(\sigma_{-N-K},\cdots,\sigma_{N+K})\in \Sigma_0: \sigma_0 = 0,~|\sigma_{\pm(k+1)}-\sigma_{\pm k}|\le 2, 0\le k\le N+K-1\right\}.
\end{equation}
We contend that
\begin{equation}
\label{eqn:net-construction}
\Net_\Delta = \{f_\sigma: \sigma\in\Sigma\}.
\end{equation}
is a $\Delta(\mu, \eta)$-net for $\Lip_L$ w.r.t. metric $\rho_\infty$ (note that $f_\sigma$ is only $O(L)$-Lipschitz, thus $\Net_\Delta$ may not be a subset of $\Lip_L$, but this is of no consequence to our proof). Proving this is the goal of the next step.

We conclude this step of the proof by establishing the item (i) of Lemma~\ref{lem:D_bounds}. Suppose $\sigma\in\Sigma$. For any fixed $\sigma_{\pm k}$, there are at most five possible values of $\sigma_{\pm (k+1)}$ due to the constraint $|\sigma_{\pm(k+1)}-\sigma_k|\le 2$. But $\sigma_0$ is already fixed, thus the number of elements in $\Sigma$ cannot exceed $5^{2(N+K)}$. Since $\Net_\Delta$ is an image of $\Sigma$, we have
\[
\log|\Net_\Delta| \le \log |\Sigma| \lesssim N+K \lesssim \frac{\log(1/\mu)}{\log(1+\eta')} + \frac{1}{\eta'}
\lesssim \frac{\log(1/\mu)}{\eta'} = \frac{4\log(1/\mu)}{\eta},
\]
where we used $\eta'=\eta/4$ which implies $\eta'<1/8$ and hence $\log(1+\eta')\ge \eta'/4$. This proves the item (i) of Lemma~\ref{lem:D_bounds} as desired.

\subsubsection{Verifying $\Net_\Delta$ is a $\Delta(\mu,  \eta)$-net.}
\label{subsubsec:apd:net_proof}
For each $f\in\Lip_L$, we can define $\sigma(f)\in\Sigma_0$ as
\[
\sigma(f)_{\pm k} = \begin{cases}
\left\lceil\frac{f(\pm z_k)}{L\eta' \max(z_k, z_N)}\right\rceil,\quad &k\in[N+K], \\
0, \quad &k=0.
\end{cases}
\]
By definition, one has
\begin{equation}
\label{eqn:sigma_bound_crude}
\frac{f(\pm z_k)}{L\eta'\max(z_k, z_N)}\le \sigma(f)_{\pm k} < \frac{f(\pm z_k)}{L\eta'\max(z_k, z_N)} + 1, \quad k\in[N+K].
\end{equation}
It suffices to verify that $\sigma(f)\in\Sigma$ and $\rho_\infty(f, f_{\sigma(f)}) \le \Delta(\mu, \eta)$. 

\paragraph{Verifying $\sigma(f)\in\Sigma$.}
We need to show $|\sigma(f)_{\pm k}| \le N$ and $|\sigma(f)_{\pm (k+1)} - \sigma(f)_{\pm k}| \le 2$. The former follows easily from the definition of $\sigma(f)$ and the fact that $|f(\pm z_k)|\le Lz_k$ (which in turn follows from $f\in\Lip_L$). For the former, we distinguish three cases: (i) $k=0$, (ii) $1\le k\le N-1$, (iii) $N\le k \le N+K$. 

For case (i), \eqref{eqn:sigma_bound_crude} implies
\[
|\sigma(f)_{\pm 1}| \le \frac{|f(\pm z_1)|}{L\eta' z_N} + 1 \le \frac{Lz_1}{L\eta' z_N} + 1 = \frac{1}{N\eta'} + 1 \le 2,
\]
where the equality follows from $z_1=\frac1N z_N$, and the last inequality follows from $N=\lceil 1/\eta'\rceil\ge 1/\eta'$. Since $\sigma(f)_0=0$, this implies $|\sigma(f)_{\pm 1} - \sigma(f)_0|\le 2$ as claimed. 

For case (ii), invoke \eqref{eqn:sigma_bound_crude} again to obtain
\begin{align*}
|\sigma(f)_{\pm(k+1)}-\sigma(f)_{\pm k}|
& \le \left|\frac{f(\pm z_{k+1})}{L\eta' z_N} - \frac{f(\pm z_{k})}{L\eta' z_N}\right| + 1
\\
& = \frac{|f(\pm z_{k+1}) - f(\pm z_k)|}{L\eta' z_N} + 1
\\
& \le \frac{L(z_{k+1} - z_k)}{L\eta' z_N} + 1
\\
& = \frac{1}{N\eta'} + 1 \le 2,
\end{align*}
where the third line follows from $f$ being $L$-Lipschitz and the last line follows from $z_{k+1}=\frac{k+1}N z_N$, $z_k=\frac{k}{N} z_N$.

For case (iii), \eqref{eqn:sigma_bound_crude} implies
\eqref{eqn:sigma_bound_crude} implies
\begin{align*}
|\sigma(f)_{\pm(k+1)}-\sigma(f)_{\pm k}|
& \le \left|\frac{f(\pm z_{k+1})}{L\eta' \max(z_{k+1}, z_N)} - \frac{f(\pm z_{k})}{L\eta' \max(z_{k}, z_N)}\right| + 1
\\
& = \left|\frac{f(\pm z_{k+1})}{L\eta' (1+\eta')z_k} - \frac{(1+\eta')f(\pm z_{k})}{L\eta'(1+\eta')z_{k}}\right| + 1
\\
& \le \frac{|f(\pm z_{k+1}) - f(\pm z_k)|}{L\eta' (1+\eta')z_k} + \frac{|f(\pm z_{k})|}{L(1+\eta')z_{k}} + 1
\\
& \le \frac{L(z_{k+1}-z_k)}{L\eta' (1+\eta')z_k} + \frac{Lz_k}{L(1+\eta')z_{k}} + 1
\\
& = \frac{2}{1+\eta'} + 1 < 3,
\end{align*}
where we used $z_{k+1}=(1+\eta')z_k$ by \eqref{eqn:z_prop} and the fourth line used Lipschitz continuity of $f$ and $f(0)=0$. This implies $|\sigma(f)_{\pm(k+1)}-\sigma(f)_{\pm k}|\le 2$ since $\sigma(f)_k$'s are integers. In conclusion, $\sigma(f)\in\Sigma$, as desired.

\paragraph{Verifying $\rho_\infty(f, f_{\sigma(f)}) \le \Delta(\mu, \eta)$.}
By Lemma~\ref{lem:D_bounds}, we need to bound $\|f-f_{\sigma(f)}\|_{L^\infty(I_j)}$. 
We begin by noting that for the maximum of the absolute value of an affine-linear function in an interval can only be attained at the endpoints of the interval. Since $f_{\sigma(f)}$ is by definition affine-linear on the interval $[z_k, z_{k+1}]$, we have
\begin{equation*}
\sup_{z\in [z_k, z_{k+1}]} |f(z) - f_{\sigma(f)}(z)| \le \sup_{z\in [z_k, z_{k+1}]} \max\left(|f(z) - f_{\sigma(f)}(z_k)|,~|f(z) - f_{\sigma(f)}(z_{k+1})|\right).
\end{equation*}
From \eqref{eqn:sigma_bound_crude} and the construction of $f_{\sigma(f)}$, we have 
\[
|f(z_k) - f_{\sigma(f)}(z_k)| \le L\eta'\max(z_k, z_N), \quad 0\le k \le N+K.
\]
On the other hand, Lipschitz continuity of $f$ implies, 
\[
\sup_{z\in[z_k, z_{k+1}]} \max\left( |f(z) - f(z_k)|,~|f(z) - f(z_{k+1})| \right) \le L(z_{k+1}-z_k).
\]
Combining the above three inequalities yields
\[
\sup_{z\in [z_k, z_{k+1}]} |f(z) - f_{\sigma(f)}(z)| \le L\eta'\max(z_k, z_N) + L(z_{k+1} - z_k).
\]
By an argument verbatim to the above, we also have
\[
\sup_{z\in [-z_{k+1}, -z_k]} |f(z) - f_{\sigma(f)}(z)| \le L\eta'\max(z_k, z_N) + L(z_{k+1} - z_k).
\]
The above two inequalities hold for all $0\le k\le N+K$, therefore for any $k\in[N+K]$, we have
\begin{align}
\|f - f_{\sigma(f)}\|_{L^\infty([-z_k, z_k])} 
& = \max_{l\in[k-1]}\sup_{z\in[-z_{l+1},-z_l]\cup[z_l, z_{l+1}]} |f(z) - f_{\sigma(f)}(z)|
\nonumber\\
& \le \max_{l\in[k-1]}\left(L\eta'\max(z_l, z_N) + L(z_{l+1} - z_l)\right)
\nonumber\\
& = \begin{cases}
    2L\eta' z_k, & N+1\le k\le N+K, \\
    L\left(\eta' + \frac1N\right)z_N , & 1\le k\le N,
\end{cases}
\label{eqn:net_error_tmp}
\end{align}
where the last equality follows from straightforward computations by the construction of $z_k$'s.

We are now ready to bound $\rho_\infty(f, f_\sigma(f))$. Recall Lemma~\ref{lem:D_bounds} and the notation $I_j=[-R\|\bx_{j}\|, R\|\bx_{j}\|]$ there. Since we have assumed without loss of generality that $\bX$ has orthonormal columns, $\bx_j$ being one of its row has norm no more than $1$. Therefore $R\|\bx_{j}\|\le R=z_{N+K}$. 
In virtue of \eqref{eqn:z_prop}, for any $j\in[n]$ such that $\|\bx_{j}\|\ge z_N/R$, there is some $k_j\ge N$ such that $\frac12 z_{k_j}\le R\|\bx_{j}\|\le z_{k_j}$, hence $I_j\subset[-z_{k_j}, z_{k_j}]$. Invoking \eqref{eqn:net_error_tmp}, we obtain, for any $j\in[n]$ such that $\|\bx_{j}\|\ge z_N/R$, that
\begin{align}
\frac{1}{p_j} \|f - f_{\sigma(f)}\|_{L^\infty(I_j)}^2 
&\le \frac{1}{p_j} \|f - f_{\sigma(f)}\|_{L^\infty([-z_{k_j}, z_{k_j}])}^2 
\nonumber\\
& \le \frac{1}{p_j}(2L\eta' z_{k_j})^2 
\nonumber\\
& \le \frac{1}{p_j}(2L\eta'\cdot  2R\|\bx_{j}\|)^2 
\nonumber\\
& = \frac{d}{\|\bx_{j}\|^2}\cdot \eta^2 L^2 R^2 \|\bx_{j}\|^2 = \eta^2 L^2 R^2 d,
\label{eqn:L_infty_bound_for_large_rows}
\end{align}
where the penultimate equality follows from $\eta'=\eta/4$ and $p_{j}=\|\bx_{j}\|^2/d$. 
On the other hand, for $j\in[n]$ such that $\|x_j\|<z_N/R$, one may invoke \eqref{eqn:net_error_tmp} again to obtain
\begin{align*}
\frac{1}{p_j} \|f - f_{\sigma(f)}\|_{L^\infty(I_j)}^2 
&\le  \frac{1}{p_j} \|f - f_{\sigma(f)}\|_{L^\infty([-z_{N}, z_{N}])}^2 \\
& \le \frac{1}{p_j} \left( L\Big(\eta'+\frac1N\Big) z_{N} \right)^2 \\
& \le \frac{1}{p_j}(2L\eta' \mu R)^2 \\
& = \frac14 \eta^2 L^2 R^2\cdot \frac{\mu^2}{p_j},
\end{align*}
where the penultimate line follows from $N=\lceil 1/\eta'\rceil\ge 1/\eta'$ and from \eqref{eqn:z_crit}, and the last line used $\eta'=\eta/4$. 
Summing up the above two inequalities, we have, for any $j\in[n]$, regardless of how large $\|\bx_j\|$ is, that
\[
\frac{1}{p_j} \|f - f_{\sigma(f)}\|_{L^\infty(I_j)}^2 
\le \eta^2 L^2 R^2 d + \frac14 \eta^2 L^2 R^2\cdot \frac{\mu^2}{p_j}.
\]
Now we apply Lemma~\ref{lem:D_bounds} to obtain
\begin{align*}
\rho_\infty(f, f_{\sigma(f)}) 
&\le LR \left\|\sum_{i=1}^m \frac{1}{p_{j_i}} \bx_{j_i}\bx_{j_i}^\top\right\|^{1/2} 
\left( \sum_{i=1}^m \frac{1}{p_{j_i}} \|f - f_{\sigma(f)} \|^2_{L^\infty(I_{j_i})} \right)^{1/2}
\\
&\le LR \left\|\sum_{i=1}^m \frac{1}{p_{j_i}} \bx_{j_i}\bx_{j_i}^\top\right\|^{1/2} \cdot \eta LR \left( \sum_{i=1}^m \Big(d + \frac{\mu^2}{4p_{j_i}}\Big)  \right)^{1/2}
\\
&\le \eta L^2 R^2 \left\|\sum_{i=1}^m \frac{1}{p_{j_i}} \bx_{j_i}\bx_{j_i}^\top\right\|^{1/2} \left( md + \sum_{i=1}^m\frac{\mu^2}{p_{j_i}}  \right)^{1/2}
\\
& = \Delta(\mu, \eta),
\end{align*}
as claimed.

\subsubsection{Bounding the Dudley integral of $\Net_\Delta$}
\label{subsubsec:apd:dudley_bound}
Recall that we need to bound $\int_0^\infty\sqrt{\log\Net(\Net_\Delta, \rho_2, \eps)}~\mathrm{d}\eps$, which amounts to bounding the entropy of the metric space $(\Net_\Delta, \rho_2)$. This will be achieved by embedding this metric space into $(\Lip_1, L^\infty([-1, 1]))$, which is again a space of Lipschitz functions yet with metric $L^\infty([-1, 1])$ instead of $\rho_\infty$. The entropy of this latter space has a well-known bound. 
The embedding is achieved by the following two lemmas.

\begin{lemma}
\label{lem:net_embed_discrete}
With the notations $f_\sigma, \sigma\in\Sigma$ as in Section~\ref{subsubsec:apd:net_construction}, we have, for all $\sigma,\sigma'\in\Sigma$, that
\[
\rho_2(f_{\sigma}, f_{\sigma'}) 
\lesssim  
\eta L^2 R^2 \sqrt d \left\|\sum_{i=1}^m \frac{1}{p_{j_i}} \bx_{j_i}\bx_{j_i}^\top\right\|^{1/2} 
\|\sigma - \sigma'\|_\infty + L^2 R^2 d \left(\sum_{i=1}^m \mathbf{1}_{\|\bx_{j_i}\|<\mu}\right)^{1/2} \rho_\delta(\sigma, \sigma').
\]
Here $\|\cdot\|_\infty$ denotes the $\ell^\infty$-norm of vectors, and $\rho_\delta$ is, by abuse of notation, the Dirac metric on $\Sigma$, defined as
\[
\rho_\delta(\sigma, \sigma') = \begin{cases}
    1, & \sigma=\sigma',\\
    0, & \sigma\ne\sigma'.
\end{cases}
\]
\end{lemma}
\begin{proof}
By the construction of $f_\sigma, f_{\sigma'}$ in~\eqref{eqn:net-pl}, we can verify in a similar way as we proved \eqref{eqn:L_infty_bound_for_large_rows} that
\begin{equation*}
\sup_{i\in[m]} \frac{ \|f_\sigma - f_{\sigma'}\|_{ L^\infty(I_{j_i}) } \mathbf{1}_{ \|\bx_{j_i}\|\ge \mu}}{\sqrt{p_{j_i}}} \lesssim \eta L R \sqrt{d} \|\sigma - \sigma'\|_{\infty}. 
\end{equation*}
The conclusion then follows from the definition of $\rho_2$.
\end{proof}

\begin{lemma}
\label{lem:discrete_embed_lip}
For $\sigma\in\Sigma$, define a piecewise linear function $g_\sigma$ on $[-1, 1]$ by
\[
g_\sigma\left(\frac{\pm k}{N+K}\right) = \frac{\sigma_{\pm k}}{2(N+K)}.
\]
At a point other than $\frac{\pm k}{N+K}$, the value of $g_\sigma$ is determined by linear interpolation. We have $g(0)=0$, $g$ is $1$-Lipschitz, and 
\[
\|\sigma-\sigma'\|_{\infty} =  2(N+K)\|g_\sigma - g_{\sigma'}\|_{L^\infty([-1, 1])},
\quad \sigma,\sigma'\in\Sigma.
\]
\end{lemma}
\begin{proof}
The equality follows easily from the fact that the supremum of the absolute value of an affine-linear function in an interval can only be attained at the endpoints of the interval. In particular, for any $k\in\mathbb Z$, $-(N+K)\le k\le N+K-1$, we have
\begin{align*}
\sup_{z\in[\frac{k}{N+K}, \frac{k+1}{N+K}]} |g_\sigma(z) - g_{\sigma'}(z)|
= \frac{1}{2(N+K)} \max\left( |\sigma_{k+1}-\sigma'_{k+1}|, |\sigma_k - \sigma'_k|\right).
\end{align*}
Taking union over all such $k$ completes the proof for the equality. It remains to check $g$ is $1$-Lipschitz. As $g$ is piecewise linear, it suffices to check the slope on each of the defining interval lies in $[-1,1]$, i.e.,
\[
(N+K)\left| g_\sigma\left(\frac{\pm (k+1)}{N+K}\right) - g_\sigma\left(\frac{\pm k}{N+K}\right) \right| \le 1,
\]
which, by definition of $g_\sigma$, is equivalent to
\[
|\sigma_{\pm(k+1)} - \sigma_{\pm k}| \le 2.
\]
But this is certainly true as $\sigma\in\Sigma$. This completes the proof.
\end{proof}

Using the embeddings given by Lemma~\ref{lem:net_embed_discrete} and Lemma~\ref{lem:discrete_embed_lip}, we can reduce the goal of bounding the entropy of $(\Net_\Delta, \rho_2)$ to bounding the entropies of $(\Lip_1, L^\infty([-1, 1])$ and of $(\Sigma, \rho_\delta)$.
\begin{corollary}
\label{cor:net_dudley_tmp}
Assume $m\ge Cd\log d$. We have
\begin{align*}
\int_0^\infty\sqrt{\log\Net(\Net_\Delta, \rho_2, \eps)}~\mathrm{d}\eps
& \lesssim L^2 R^2 \sqrt{d}\log\left(\frac{1}{\mu}\right) \cdot\left\|\sum_{i=1}^m \frac{1}{p_{j_i}} \bx_{j_i}\bx_{j_i}^\top\right\|^{1/2} \int_0^\infty\sqrt{\log\Net(\Lip_1, L^\infty([-1, 1]), \eps)}~\mathrm{d}\eps
\\
& + L^2 R^2 d \left(\sum_{i=1}^m \mathbf{1}_{\|\bx_{j_i}\|<\mu}\right)^{1/2} \int_0^\infty\sqrt{\log\Net(\Sigma, \rho_\delta, \eps)}~\mathrm{d}\eps.
\end{align*}
\end{corollary}
\begin{proof}
Combining Lemma~\ref{lem:net_embed_discrete} and Lemma~\ref{lem:discrete_embed_lip}, we have
\begin{align*}
\rho_2(f_\sigma, f_{\sigma'})
& \lesssim 
\eta(N+K) L^2 R^2 \sqrt{d}\left\|\sum_{i=1}^m \frac{1}{p_{j_i}} \bx_{j_i}\bx_{j_i}^\top\right\|^{1/2} \|g_\sigma - g_{\sigma'}\|_{L^\infty([-1, 1])}
+ L^2 R^2 d \left(\sum_{i=1}^m \mathbf{1}_{\|\bx_{j_i}\|<\mu}\right)^{1/2} \rho_\delta(\sigma, \sigma')
\\
& \lesssim L^2 R^2 \sqrt{d}\log\left(\frac{1}{\mu}\right) \cdot\left\|\sum_{i=1}^m \frac{1}{p_{j_i}} \bx_{j_i}\bx_{j_i}^\top\right\|^{1/2} \|g_\sigma - g_{\sigma'}\|_{L^\infty([-1, 1])}
+ L^2 R^2 d \left(\sum_{i=1}^m \mathbf{1}_{\|\bx_{j_i}\|<\mu}\right)^{1/2} \rho_\delta(\sigma, \sigma'),
\end{align*}
where we used $N\lesssim 1/\eta$ and $K\lesssim \frac{1}{\eta}\log\frac{1}{\mu}$. If we regard all three metrics $\rho_2(f_\sigma, f_{\sigma'})$, $\|g_\sigma - g_{\sigma'}\|_{L^\infty([-1, 1])}$, $\rho_\delta(\sigma, \sigma')$ as metrics on $\Sigma$, the desired conclusion  readily follows from Lemma~\ref{lem:sublinear-gamma2}.
\end{proof}

It remains to bound the entropies of $(\Lip_1, L^\infty([-1, 1])$ and of $(\Sigma, \rho_\delta)$. 
The former is well-known, presented in the next lemma, while the latter can be computed easily.
\begin{lemma}[Entropy bound of Lipschitz function class, \cite{talagrand2022upper}]
\label{lem:lip_entropy}
We have
\[
\log\Net(\Lip_1, L^\infty([-1, 1]), \eps) \lesssim \frac{1}{\eps}, \quad \eps>0.
\]
\end{lemma}

We have now collected all the ingredients to prove the item (ii) of Lemma~\ref{lem:D-infty-net}.
\begin{proof}[Proof of the item (ii) of Lemma~\ref{lem:D-infty-net}]
Since the diameter of $\Lip_1$ with respect to $L^\infty([-1, 1])$ is $1$, we deduce immediately from the Lemma~\ref{lem:lip_entropy} that
\begin{equation*}
\int_0^\infty\sqrt{\log\Net(\Lip_1, L^\infty([-1, 1]), \eps)}~\mathrm{d}\eps
= \int_0^1\sqrt{\log\Net(\Lip_1, L^\infty([-1, 1]), \eps)}~\mathrm{d}\eps
\lesssim \int_0^1\sqrt{\frac1{\eps}}~\mathrm{d}\eps\lesssim 1.
\end{equation*}

Turning to bound the entropy of $(\Sigma,\rho_\delta)$, we simply note that $\rho_\delta\le 1$, hence
\begin{equation*}
\int_0^\infty\sqrt{\log\Net(\Sigma, \rho_\delta, \eps)}~\mathrm{d}\eps
= \int_0^1\sqrt{\log\Net(\Sigma, \rho_\delta, \eps)}~\mathrm{d}\eps \lesssim \int_0^1\sqrt{\log|\Sigma|}~\mathrm{d}\eps
\lesssim \sqrt{\frac{\log(1/\mu)}{\eta}},
\end{equation*}
where the last inequality used the conclusion of item (i) in Lemma~\ref{lem:D-infty-net}, proved in Section~\ref{subsubsec:apd:net_construction}.

Plug the above two inequalities into Corollary~\ref{cor:net_dudley_tmp} to obtain
\begin{align*}
\int_0^\infty\sqrt{\log\Net(\Net_\Delta, \rho_2, \eps)}~\mathrm{d}\eps
& \lesssim L^2 R^2 \sqrt{d}\log\left(\frac{1}{\mu}\right) \cdot\left\|\sum_{i=1}^m \frac{1}{p_{j_i}} \bx_{j_i}\bx_{j_i}^\top\right\|^{1/2} 
\\
& + L^2 R^2 d \left(\sum_{i=1}^m \mathbf{1}_{\|\bx_{j_i}\|<\mu}\right)^{1/2} \sqrt{\frac{\log(1/\mu)}{\eta}}.
\end{align*}
The desired conclusion follows taking expectations on both sides, using \eqref{eqn:matrix-chernoff} and Cauchy-Schwarz to bound the expectation of the first term on the right hand side, and using the following fact to bound the expectation of the second term on the right hand side:
\begin{align*}
\bbE \left(\sum_{i=1}^m \mathbf{1}_{\|\bx_{j_i}\|<\mu}\right)^{1/2} 
 \le \left(\bbE \sum_{i=1}^m \mathbf{1}_{\|\bx_{j_i}\|<\mu}\right)^{1/2} 
& = \left(m \sum_{j=1}^n p_j\mathbf{1}_{\|\bx_{j}\|<\mu}\right)^{1/2} \\
& = \left(m \sum_{j=1}^n \frac{\|\bx_j\|^2}{d}\mathbf{1}_{\|\bx_{j}\|<\mu}\right)^{1/2} \\
& \le \left(m \sum_{j=1}^n \frac{\mu^2}{d}\right)^{1/2} \\
& = \mu\sqrt{\frac{nm}{d}}.
\end{align*}
\end{proof}

\subsubsection{Bounding $\E\Delta$.}
\label{subsubsec:apd:Delta_expectation}
Recall the Chernoff bound \eqref{eqn:matrix-chernoff}, we have by Cauchy-Schwarz inequality that
\begin{align*}
\E\Delta \lesssim \eta L^2 R^2 \sqrt{m} \left(md + \E\sum_{i=1}^m \frac{\mu^2}{p_{j_i}}\right)^{1/2}.
\end{align*}
Note that
\begin{align*}
\E\frac{1}{p_{j_i}} = \sum_{j\in[n]} p_j \frac{1}{p_j} = n,
\end{align*}
thus 
\begin{align*}
\E\Delta \lesssim \eta L^2 R^2 \sqrt{m} \left(md + m\mu^2 n\right)^{1/2} \lesssim L^2 R^2 m \sqrt{d} \cdot \sqrt{1+\frac{\mu^2 n}{d}},
\end{align*}
as claimed.

\end{document}